\documentclass[letterpaper]{article} 
\usepackage{aaai2027}  
\usepackage[hyphens]{url}  
\usepackage{graphicx} 
\usepackage{natbib}  
\usepackage{caption} 
\DeclareCaptionStyle{ruled}{labelfont=normalfont,labelsep=colon,strut=off} 
\usepackage{booktabs}
\usepackage{amsmath,amssymb}
\usepackage{algorithm}
\usepackage{algorithmic}
\usepackage{tikz}
\usetikzlibrary{decorations.pathreplacing,calc}
\newtheorem{proposition}{Proposition}

\title{Teacher-Anchored Selection of Post-Training Quantized Models
under Domain Shift}

\author{
Alejandro Rodriguez Dominguez,\textsuperscript{\rm 1,2,3}
Muhammad Shahzad,\textsuperscript{\rm 1}
Xia Hong\textsuperscript{\rm 1}
}

\affiliations{
\textsuperscript{\rm 1}Department of Computer Science,
University of Reading, Reading, United Kingdom\\
\textsuperscript{\rm 2}Quantitative Analysis and Artificial Intelligence Department,
Miralta Finance Bank S.A., Madrid, Spain\\
\textsuperscript{\rm 3}Lecturer, Albert School, Madrid, Spain
}

\begin{document}
\maketitle

\begin{abstract}
Compressing a trained model yields a family of deployment candidates, and under domain
shift the most compressed one need not be the one to deploy. We study selection over such a family, with candidates and teacher fixed and target
labels absent or scarce. Two findings organize the label-free case. Minimum teacher distortion behaves almost as a constant rule, selecting the same
eight-bit, per-channel, unclipped configuration in every run, which does not minimize
empirical target cross-entropy. Established estimators divide sharply: in the overconfident-collapse regime of the CNN
families, confidence-based estimators order the family close to backwards, and the diagnostics that identify it need the labels the setting denies, while output-distribution
estimators match the teacher-relative anchor and on one architecture beat it. Distortion is nonetheless stable, so a supervised term can move selection away from it.
Combining the two, we give exact quadratic identities for a canonical quadratic analogue
of the family. We also show that under symmetric corruption the label-dependent part of a
criterion linear in the label indicator is multiplied by one common factor whenever its
coefficient sums are candidate-invariant, a class holding teacher contrasts and accuracy
but not cross-entropy. These characterize the score's components without bounding selection regret. Across one hundred and thirty-four candidate families, one per independently trained
convolutional or Vision Transformer teacher, anchoring reduces mean regret at the
smallest label budget in every setting, an advantage that fades beyond twenty-five
labels.
\end{abstract}

\section{Introduction}
\label{sec:introduction}
Model compression usually produces several plausible deployment candidates rather than
one preferable model \citep{cheng2018survey}. Post-training quantization is a clear example
\citep{jacob2018quantization}, where bit widths, clipping thresholds, granularity and
layer configurations generate a family derived from one full-precision teacher
\citep{gou2021knowledge,hinton2015distilling}, and recent methods keep enlarging it by
tuning parameters per layer \citep{ranjan2026mixqvit,arai2025qep}. Under domain shift, candidates that behave alike on source or calibration data can have
very different target risks \citep{ovadia2019can,minderer2021revisiting}, so the
construction criterion need not identify the best model to deploy.

This is a deployment-time problem distinct from candidate construction: given an
already built family sharing a teacher, which candidate should be used on the target
domain? Candidates are neither retrained nor adapted during selection, and the
held-out target test set is used only for evaluation. The objective is the candidate
with lowest target cross-entropy. Operational constraints such as memory, latency or
energy are assumed to have restricted the family beforehand, so we study selection within an admissible family, not the broader problem of trading
risk against cost.

Two signals are available and neither suffices alone. Distance to the shared teacher
needs no labels and is stable across resamples \citep{liu2023pdquant}, but on standard
grids the candidate closest to its teacher is almost always the least-compressed one,
so its ranking is nearly fixed in advance. A small labeled target sample measures the objective
directly but at the budgets a practitioner will spend is dominated by variance
\citep{okanovic2025all}, and corrupted labels degrade its ordering further
\citep{chen2021robustness}. The failures are complementary, one biased but stable and the other unbiased but noisy,
which is where shrinkage should help.

Anchored selection makes that shrinkage explicit: the label-free distortion is the
base score, the supervised loss enters as a weighted correction, and the weight is
chosen by cross-validation inside the labeled sample, so the labels move the choice
away from the teacher only as far as they support. The construction admits exact
analysis. Quadratic identities relate the members of a canonical quadratic analogue of the family
and make explicit what each supervised term adds beyond the teacher distance; the
selector we evaluate replaces the squared distortion with KL and tunes its coefficient,
so the identities characterize components without governing regret.
Under symmetric corruption the label-dependent part of a criterion linear in the label indicator
attenuates by a single common factor whenever its coefficient sums do not vary across
candidates. That covers validation accuracy as well as the Brier score
\citep{gneiting2007strictly}, so the result describes a class rather than favouring any
member of it.

Our contributions are to formulate deployment-time selection over a fixed family sharing
a teacher and to characterize its label-free baselines, showing that the teacher-relative rule is close
to degenerate, that confidence-based estimators can fail catastrophically
under overconfident quantization collapse in a way that depends on a regime not
identifiable without target labels, and that output-distribution estimators remain
competitive with the anchor and on one architecture outperform it; to develop the anchored family, together with quadratic
identities for its canonical quadratic analogue; to prove the
attenuation result; and to show that anchoring improves on direct validation at the
smallest label budgets, an advantage that fades as the supervised estimate becomes
reliable alone.

\section{Related Work}
\label{sec:related}
\paragraph{Performance estimation under shift.}
Without target labels, methods estimate performance from confidence, entropy,
calibrated thresholds, class dispersity or prediction-matrix statistics
\citep{hendrycks2017baseline,garg2022atc,guillory2021doc,deng2023nuclear,
xie2024mano}. These target absolute performance, a weaker requirement than ours:
selection needs only the ordering, but needs it right between candidates that
differ slightly, so small estimation errors reverse it. A second line reads
agreement between independently trained predictors as a proxy for accuracy
\citep{jiang2022assessing,baek2022agreementontheline}, and concurrent work
supplies a foundation model as external reference \citep{li2026frap}. Both require a reference the practitioner must provide; ours already contains one,
since every candidate derives from the same teacher. Semi-Supervised Model Evaluation is closest in the
supervision it assumes \citep{shanmugam2025ssme}, but models the joint scores of
several classifiers rather than their structured relation to a common parent.

\paragraph{Shared teachers and compressed-model construction.}
The teacher those methods lack is standard in compression. Distillation transfers
a teacher's predictive distribution to a smaller model
\citep{hinton2015distilling,tian2020contrastive}, and post-training quantization
methods use prediction differences, layer sensitivities or propagated quantization
errors to choose quantization parameters
\citep{liu2023pdquant,ranjan2026mixqvit,arai2025qep}. There the teacher guides construction and the resulting candidate is the answer; our problem begins with the family fixed. Teacher
distance is the signal to carry across that boundary, and carrying it alone is not
enough.

\paragraph{Selection with scarce or unreliable labels.}
Spending labels raises the question of what to compute from very few. Direct validation is the natural baseline, and label-efficient selection has been
approached by choosing which examples to annotate \citep{okanovic2025all}; ours is
complementary, since every selector receives the same examples and budget. Which loss is computed matters, because cross-entropy, Brier score and
accuracy induce different rankings once calibration and confidence shift
\citep{guo2017calibration,ovadia2019can,minderer2021revisiting}. The Brier and logarithmic scores are standard proper scoring rules
\citep{gneiting2007strictly}; we introduce no new one, asking instead which to anchor. Under multiclass class-conditional noise validation accuracy can retain selection
information \citep{chen2021robustness}; our attenuation result explains why it survives
corruption, and identifies the class of linear criteria that share that behavior.
Anchored selection sits at the intersection of these three. Performance estimation
normally imports a reference from outside; here construction has already supplied one,
the teacher, which we combine with a supervised term small enough that the
scarce-label concerns apply directly.

\section{Anchored Selection and Its Identities}
\label{sec:framework}
Let $f$ be a fixed teacher with predictive distribution $p_f(\cdot\mid x)$ and
$\mathcal G=\{g_1,\ldots,g_M\}$ a finite family of candidates with predictions
$p_g(\cdot\mid x)$. The framework assumes only that all candidates share the
same reference teacher. On the shifted target domain we observe an unlabeled
pool $\mathcal U$ and, depending on the regime, a labeled sample
$\mathcal S=\{(x_i,y_i)\}_{i=1}^{n}$. A disjoint test set $\mathcal T$ is used
only for evaluation and defines
$L_{\mathrm{test}}(g)=|\mathcal T|^{-1}\sum_{(x,y)\in\mathcal T}-\log
p_g(y\mid x)$. Selector quality is the cross-entropy selection regret
$R(\hat g)=L_{\mathrm{test}}(\hat g)-\min_{g}L_{\mathrm{test}}(g)$.

\subsection{One Family of Selectors}
\label{subsec:anchored}
Without target labels the natural rule keeps the candidate whose predictions stay
closest to the teacher's, by the average KL distortion over the unlabeled pool,
\[
D(g)=\frac{1}{|\mathcal U|}\sum_{x\in\mathcal U}
\mathrm{KL}\!\left(p_f(\cdot\mid x)\,\|\,p_g(\cdot\mid x)\right)
\]
which gives $\hat g_D=\arg\min_g D(g)$. It measures how
much a candidate changes the teacher's predictions, not whether the change is
favorable.

Every selector we study adds one function of the labeled sample to that base
score:
\begin{equation}
\label{eq:family}
S(g)=D(g)+\alpha\,\Phi_{\mathcal S}(g),\qquad \alpha\ge 0
\end{equation}
so the members differ only in the choice of $\Phi_{\mathcal S}$, and $\alpha=0$ recovers
pure distortion. Three choices matter here.

\begin{itemize}\itemsep2pt
\item \textbf{CE-combo} takes $\Phi_{\mathcal S}=\widehat{\mathrm{CE}}_{\mathcal S}$,
the empirical cross-entropy on the labeled sample. Large $\alpha$ recovers direct
validation, so \eqref{eq:family} interpolates between the two label regimes.
\item \textbf{Alignment} takes $\Phi_{\mathcal S}=-A_{\mathcal S}$, where
$A_{\mathcal S}=|\mathcal S|^{-1}\sum\langle p_g(x)-p_f(x),\,e_y-p_f(x)\rangle$
scores how far a candidate moves \emph{from} the teacher prediction \emph{toward}
the observed class, rather than how far it moves at all. The sign is flipped
because favorable movement should lower the score.
\item \textbf{Permutation control} takes the same $A_{\mathcal S}$ after permuting
the observed labels across the sample. It keeps the statistic's scale and destroys only the label--input correspondence. What
it isolates is therefore what the direction contributes beyond the shrinkage.
\end{itemize}

\noindent
Tables abbreviate these consistently as Dist.\ for the label-free anchor, Val-CE and
Val-Acc for direct validation on cross-entropy and on accuracy, CE-combo, Align,
Perm.\ for the permutation control and Teach.\ for the teacher component.
The coefficient is always chosen by cross-validation inside $\mathcal S$ with $\alpha=0$
in the grid, so no selector receives labels the others do not, and each can revert to
the anchor when the data do not support leaving it.
Figure~\ref{fig:collapse}(a) shows the geometry.

\subsection{What the Family Implies}
\label{subsec:properties}
Two consequences of writing the selectors this way are used later; derivations are
in Supplement~\ref{supp:identities}.

\emph{Alignment separates into a label part and a label-free part.} Because
$A_{\mathcal S}$ is linear in the observed class, $A_{\mathcal S}=
A_{\mathcal S}^{\mathrm{lab}}+A_{\mathcal S}^{\mathrm{teach}}$, where only the
first uses $e_y$. Retaining the second alone gives the \emph{teacher
component}. In the fully quadratic identity the corresponding teacher-only criterion
reduces to predictive-concentration ranking; the implemented ablation, which keeps KL
distortion and a cross-validated coefficient, is related but not identical. The selector is
not label-free: its coefficient is still cross-validated on the labeled sample, so only
the score component is label-independent. We evaluate it as an ablation.

\emph{Label information decays at one rate for a class of criteria.} Write a
label-dependent criterion as $a_g(x)^\top e_y$. Under symmetric $K$-class corruption
that retains the clean label with probability $1-\eta$,
\[
\mathbb E\!\left[a_g^\top e_{\tilde y}\mid y\right]
=\left(1-\tfrac{K}{K-1}\eta\right)a_g^\top e_y
+\tfrac{\eta}{K-1}\,\mathbf 1^\top a_g ,
\]
so a single multiplicative factor governs candidate \emph{comparisons} exactly when
the offset $\mathbf 1^\top a_g$ does not vary across candidates. That condition holds
for teacher contrasts such as $p_g-p_f$, whose coordinates sum to zero, and for
accuracy indicators, whose coordinates sum to one, giving
$\mathbb E[\widehat{\mathrm{Acc}}^{\mathrm{noisy}}]
=(1-\frac{K}{K-1}\eta)\widehat{\mathrm{Acc}}+\frac{\eta}{K-1}$. It does not hold in
general: the coefficient vector of cross-entropy, $-\log p_g$, has a
candidate-dependent coordinate sum. The alignment statistic and validation accuracy
therefore share one attenuation rate, and the result singles out neither of them; it
says nothing about criteria outside that class.

\section{Empirical Evaluation}
Four questions organize this section. Can a label-free rule solve the problem at all?
Does spending few labels help? Does the directional term add anything the anchoring does
not? And does the corruption analysis predict what we measure?

\label{sec:experiments}
We evaluate three target-supervision regimes: no labels, scarce clean labels, and
corrupted labels. Candidate parameters are never updated. Two resources are held fixed: the target inputs, which every selector may use in full,
and the label budget, the same $n$ examples for each supervised selector including those
its coefficient selection consumes. 



\paragraph{Settings and protocol.}
A \emph{run} is one teacher with its fixed $72$-candidate family; an \emph{execution}
repeats the pipeline with a retrained teacher on identical splits, so a shift--seed
configuration recurs once per execution with a different candidate family. There are $74$ base configurations, $15$ and $15$ on the CNNs and $17$ and $27$ on the
ViTs; since the CNN settings were executed three times and the ViT settings once, the
number of independently trained teachers, hence of candidate families, is
$45+45+17+27=134$. The statistical unit throughout is the run, so pooled CNN
comparisons have $45$ units per architecture; Supplement~\ref{supp:artifact} documents the
runs excluded from this cohort and the structural criteria used. Four settings are used throughout;
Supplement Table~\ref{tab:settings} summarizes them and Supplement Table~\ref{tab:supp_runs}
gives the full accounting including within-run repetitions. The DomainNet \citep{peng2019moment} CNN experiments use five fixed shifts at $30$
classes, from Real to Clipart, Painting and Sketch, and from Clipart and Painting to
Sketch. ResNet50 \citep{he2016deep} and MobileNetV2
\citep{sandler2018mobilenetv2} use three seeds per shift, giving $15$ base configurations each.
Their teachers come from transfer learning on ImageNet
\citep{russakovsky2015imagenet}, and each architecture is run through three
independent end-to-end executions on identical splits. The two
Vision Transformer \citep{dosovitskiy2021image} settings are DomainNet at $50$
classes ($17$ runs) and held-out-domain CIFAR-20 \citep{krizhevsky2009learning} ($27$ runs), with teachers from a continual-distillation pipeline; because teacher construction
differs between the two groups, we do not read cross-setting differences as evidence
about teacher quality alone. Selection pools contain $300$ labeled
examples in the CNN settings, $3646$ in ViT--DomainNet and $1332$ in
ViT--CIFAR-20.

Every run uses the same weights-only grid. Bit widths are
$\{2,3,4,5,6,8\}$ and clipping percentiles $\{99.0,99.5,100.0\}$, crossed with
per-tensor or per-channel scaling and an endpoint option for the first and last
quantizable layers, giving $6\times3\times2\times2=72$ candidates. Clipping percentiles come from the weight
tensors, so candidate construction consumes no target data. Coefficients are chosen by
cross-validation within $\mathcal S$, five folds when $n\ge 25$ and leave-one-out
below. The distortion term is evaluated once on the complete unlabeled target pool at every
budget; only the supervised term uses the labeled subsample of size $n$. Supplement~\ref{supp:protocol} states the full protocol,
including fold construction and seeds. Probabilities are lower-clipped at $10^{-8}$ in
the CNN pipelines and $10^{-12}$ in the ViT pipelines, following each implementation; we verify that the choice of floor changes regret by less than
$0.001$ across three floors spanning six orders of magnitude, so the magnitudes reflect the candidate
families and not the clipping.

\paragraph{Selectors.}
The three members of \eqref{eq:family} defined in
Section~\ref{subsec:anchored} are evaluated against two references that spend the
same labels: distortion alone, which minimizes $D^{\mathrm{KL}}$ over the target
examples available at that budget, and direct validation, which minimizes empirical
cross-entropy on $\mathcal S$. The teacher component of the alignment statistic is
evaluated as a further ablation. Full definitions are in
Supplement Table~\ref{tab:supp_selectors}.

\subsection{Selection without Target Labels}
\label{subsec:labelfree}
Minimum KL distortion is close to a constant rule: on the complete pool it chooses the
same configuration in all $134$ candidate families,
$8$-bit, per-channel,
unclipped, with the first and last quantizable layers in full precision. We write $\max_g L_{\mathrm{test}}(g)-\min_g L_{\mathrm{test}}(g)$ for a family's
\emph{cross-entropy spread}, the regret of choosing its worst member, which differs by a
factor of five across our settings. Distortion's mean
regret is $0.220$ on ResNet50 and $0.087$ on MobileNetV2 over the pooled
executions, so the label-free question is whether the least-compressed candidate is the empirical
cross-entropy minimizer, reproducible with no target data at all. The choice is also
near-degenerate: the two $8$-bit per-channel unclipped candidates differing only
in the endpoint option are separated by a distortion gap of order $10^{-4}$, their
ordering reverses on a $150$-example half-pool in one of five shifts in two of the
six CNN executions, and their target cross-entropies differ by $0.0015$.

Table~\ref{tab:sota} compares every selector at the smallest budget, so no row is
advantaged by labels another does not see. Regret is the mean over runs. The two blocks answer whether spending $n$ labels beats the best free option; paired
tests are between supervised selectors
(Supplement Table~\ref{tab:paired}). 

\begin{table}[t]
\centering
\scriptsize
\setlength{\tabcolsep}{1.5pt}
\caption{Selection regret, mean\,$\pm$\,standard deviation over runs. Columns are the
four architecture--dataset settings: ResNet50 (RN50) and MobileNetV2 (MNV2) on
DomainNet with $30$ classes, and Vision Transformers on DomainNet with $50$ classes
(ViT--DN) and on held-out-domain CIFAR-20 with $20$ classes (ViT--C20). \emph{Every
supervised selector here uses $n=10$ labels}, the smallest budget we evaluate, so all
columns are comparable; each gets the same ten examples, including those its
coefficient selection consumes. The label-free rules use none, so their values hold at
any $n$. Boldface marks the lowest mean per column. Row labels give the venue and year;
the corresponding references are \citet{hendrycks2017baseline} for average confidence
and entropy, \citet{garg2022atc} for ATC, \citet{deng2023nuclear} for nuclear norm,
\citet{guillory2021doc} for DOC, \citet{tu2024softmaxcorr} for SoftmaxCorr,
\citet{lu2023characterizing} for COT and \citet{liu2023pdquant} for distortion. Runs
are $45$ per CNN architecture over three executions, $17$ and $27$ on the ViT settings;
the other budgets are in Supplement Table~\ref{tab:sweep}.}
\label{tab:sota}
\begin{tabular}{@{}lcccc@{}}
\toprule
Method & RN50 & MNV2 & ViT--DN & ViT--C20 \\
\midrule
Direct validation & 0.182\,$\pm$0.06 & 0.227\,$\pm$0.13 & 0.208\,$\pm$0.12 & 0.222\,$\pm$0.10 \\
CE-combo (ours) & \textbf{0.166\,$\pm$0.09} & 0.103\,$\pm$0.07 & \textbf{0.139\,$\pm$0.14} & 0.115\,$\pm$0.06 \\
Alignment (ours) & 0.168\,$\pm$0.10 & 0.088\,$\pm$0.07 & 0.205\,$\pm$0.33 & 0.150\,$\pm$0.06 \\
\addlinespace[2pt]
SoftmaxCorr [TMLR'24] & 0.212\,$\pm$0.18 & 0.061\,$\pm$0.27 & 0.224\,$\pm$0.47 & \textbf{0.083\,$\pm$0.06} \\
COT [NeurIPS'23] & 0.257\,$\pm$0.21 & \textbf{0.007\,$\pm$0.01} & 0.278\,$\pm$0.66 & 0.098\,$\pm$0.05 \\
Nuclear norm [ICML'23] & 0.277\,$\pm$0.20 & 0.013\,$\pm$0.03 & 0.280\,$\pm$0.66 & 0.098\,$\pm$0.05 \\
DOC [ICCV'21] & 0.221\,$\pm$0.20 & 0.086\,$\pm$0.08 & 0.284\,$\pm$0.69 & 0.095\,$\pm$0.04 \\
Distortion [CVPR'23] (anch.) & 0.220\,$\pm$0.20 & 0.087\,$\pm$0.08 & 0.284\,$\pm$0.69 & 0.096\,$\pm$0.04 \\
ATC-MC [ICLR'22] & 14.489\,$\pm$3.05 & 4.461\,$\pm$2.01 & 0.284\,$\pm$0.69 & 0.095\,$\pm$0.04 \\
ATC-NE [ICLR'22] & 15.125\,$\pm$0.31 & 3.985\,$\pm$2.31 & 0.284\,$\pm$0.69 & 0.095\,$\pm$0.04 \\
Entropy [ICLR'17] & 15.125\,$\pm$0.31 & 10.633\,$\pm$1.57 & 0.792\,$\pm$1.17 & 0.572\,$\pm$0.62 \\
Avg.\ confidence [ICLR'17] & 15.125\,$\pm$0.31 & 10.654\,$\pm$1.53 & 0.919\,$\pm$1.47 & 0.442\,$\pm$0.62 \\
\bottomrule
\end{tabular}
\end{table}

\textbf{We screen the state of the art in label-free performance estimation}, spanning its four
families: average confidence and entropy \citep{hendrycks2017baseline}, the two ATC
variants \citep{garg2022atc}, nuclear norm \citep{deng2023nuclear}, DOC
\citep{guillory2021doc} and SoftmaxCorr \citep{tu2024softmaxcorr}, COT
\citep{lu2023characterizing}, and teacher distortion \citep{liu2023pdquant}.
All nine are computed from predictive distributions alone. ATC and DOC need a calibration set, which we form with the teacher's arg-max rather than
with labels, so the two rows are teacher-pseudo-label variants of the published methods
(Supplement~\ref{supp:protocol}). Direct validation cannot use the unlabeled inputs, so the natural missing comparison is
a semi-supervised criterion using both \citep{shanmugam2025ssme}; we do not evaluate
one, and our margin over direct validation should be read with that in mind.

The confidence-based group fails by two orders of magnitude: on ResNet50 average
confidence and entropy essentially select the worst candidate, and the two ATC variants
select from the same collapsed region. Figure~\ref{fig:collapse} shows why: low-bit quantization produces degenerate predictors
at chance accuracy whose confidence \emph{exceeds} that of the candidates minimizing
target cross-entropy, $1.000$ against $0.408$ on ResNet50. Any
statistic reading confidence as correctness therefore orders the family close to
backwards. The two ViT settings invert the pattern: their degenerate candidates stay
under-confident and the same estimators select well, so the cause is the degradation
regime and not the estimator. Supplement~\ref{supp:sens} shows the magnitudes do not depend on probability clipping, and
Supplement Table~\ref{tab:accreg} that the same failure appears when estimators are scored on
accuracy. The diagnostics that distinguish the two regimes in our experiments are measurable only
with target labels.

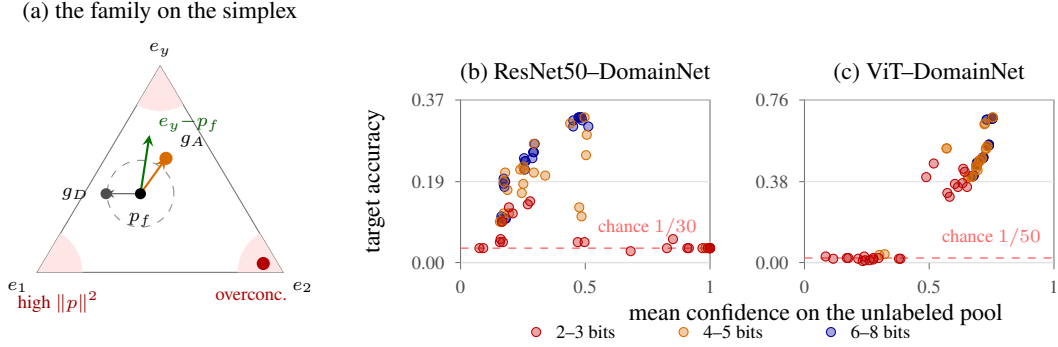
\begin{figure*}[t]
\centering
\begin{tikzpicture}[scale=1.0,font=\footnotesize,>=stealth]
\begin{scope}[shift={(0,0.42)},scale=1.30,
              every node/.style={transform shape=false,font=\scriptsize}]
\coordinate (A) at (0,0); \coordinate (B) at (2.5,0); \coordinate (C) at (1.25,2.1);
\draw[black!55] (A)--(B)--(C)--cycle;
\node[below left=0pt] at (A) {$e_1$};
\node[below right=0pt] at (B) {$e_2$};
\node[above=0pt] at (C) {$e_y$};
\begin{scope}\clip (A)--(B)--(C)--cycle;
\fill[red!10] (A) circle (0.46); \fill[red!10] (B) circle (0.46); \fill[red!10] (C) circle (0.46);
\end{scope}
\coordinate (F) at (1.05,0.80);
\draw[dashed,black!45] (F) circle (0.34);
\draw[->,thick,green!45!black] (F)--($(F)!0.46!(C)$);
\coordinate (GD) at (0.70,0.80); \coordinate (GA) at (1.31,1.16);
\draw[->,black!70] (F)--(GD);
\draw[->,orange!85!black,thick] (F)--(GA);
\fill[black] (F) circle (1.6pt);
\fill[black!70] (GD) circle (1.6pt);
\fill[orange!85!black] (GA) circle (1.9pt);
\fill[red!70!black] (2.30,0.09) circle (1.9pt);
\node[anchor=east] at (0.62,0.80) {$g_D$};
\node[anchor=south west] at (1.36,1.18) {$g_A$};
\node[anchor=north] at (1.05,0.70) {$p_f$};
\node[anchor=west,green!35!black] at (1.14,1.52) {$e_y{-}p_f$};
\node[anchor=north east,red!60!black] at (0.72,-0.10) {high $\|p\|^2$};
\node[anchor=north west,red!70!black] at (1.72,-0.10) {overconc.};
\node[anchor=south,font=\footnotesize] at (1.25,2.44) {(a) the family on the simplex};
\end{scope}
\draw[black!55,line width=0.4pt] (5.60,0.55) rectangle (8.90,2.70);
\node[anchor=south,font=\footnotesize] at (7.25,2.80) {(b) ResNet50--DomainNet};
\draw[black!55] (5.60,0.55) -- (5.60,0.46);
\node[anchor=north,font=\scriptsize,black!70] at (5.60,0.44) {0};
\draw[black!55] (7.25,0.55) -- (7.25,0.46);
\node[anchor=north,font=\scriptsize,black!70] at (7.25,0.44) {0.5};
\draw[black!55] (8.90,0.55) -- (8.90,0.46);
\node[anchor=north,font=\scriptsize,black!70] at (8.90,0.44) {1};
\draw[black!12] (5.60,0.55) -- (8.90,0.55);
\draw[black!55] (5.60,0.55) -- (5.53,0.55);
\node[anchor=east,font=\scriptsize,black!70] at (5.51,0.55) {0.00};
\draw[black!12] (5.60,1.62) -- (8.90,1.62);
\draw[black!55] (5.60,1.62) -- (5.53,1.62);
\node[anchor=east,font=\scriptsize,black!70] at (5.51,1.62) {0.19};
\draw[black!12] (5.60,2.70) -- (8.90,2.70);
\draw[black!55] (5.60,2.70) -- (5.53,2.70);
\node[anchor=east,font=\scriptsize,black!70] at (5.51,2.70) {0.37};
\draw[red!60,dashed,line width=0.5pt] (5.60,0.74) -- (8.90,0.74);
\node[anchor=south east,font=\scriptsize,red!60] at (8.85,0.77) {chance $1/30$};
\draw[blue!60!black,line width=0.4pt,fill=blue!60!black,fill opacity=0.35] (6.15,1.13) circle (1.7pt);
\draw[blue!60!black,line width=0.4pt,fill=blue!60!black,fill opacity=0.35] (6.20,1.13) circle (1.7pt);
\draw[blue!60!black,line width=0.4pt,fill=blue!60!black,fill opacity=0.35] (6.17,1.66) circle (1.7pt);
\draw[blue!60!black,line width=0.4pt,fill=blue!60!black,fill opacity=0.35] (6.18,1.55) circle (1.7pt);
\draw[blue!60!black,line width=0.4pt,fill=blue!60!black,fill opacity=0.35] (6.45,1.78) circle (1.7pt);
\draw[blue!60!black,line width=0.4pt,fill=blue!60!black,fill opacity=0.35] (6.45,1.89) circle (1.7pt);
\draw[blue!60!black,line width=0.4pt,fill=blue!60!black,fill opacity=0.35] (6.55,1.93) circle (1.7pt);
\draw[blue!60!black,line width=0.4pt,fill=blue!60!black,fill opacity=0.35] (6.56,2.01) circle (1.7pt);
\draw[blue!60!black,line width=0.4pt,fill=blue!60!black,fill opacity=0.35] (7.17,2.47) circle (1.7pt);
\draw[blue!60!black,line width=0.4pt,fill=blue!60!black,fill opacity=0.35] (7.16,2.47) circle (1.7pt);
\draw[blue!60!black,line width=0.4pt,fill=blue!60!black,fill opacity=0.35] (7.19,2.47) circle (1.7pt);
\draw[blue!60!black,line width=0.4pt,fill=blue!60!black,fill opacity=0.35] (7.20,2.47) circle (1.7pt);
\draw[blue!60!black,line width=0.4pt,fill=blue!60!black,fill opacity=0.35] (6.14,1.16) circle (1.7pt);
\draw[blue!60!black,line width=0.4pt,fill=blue!60!black,fill opacity=0.35] (6.19,1.13) circle (1.7pt);
\draw[blue!60!black,line width=0.4pt,fill=blue!60!black,fill opacity=0.35] (6.17,1.59) circle (1.7pt);
\draw[blue!60!black,line width=0.4pt,fill=blue!60!black,fill opacity=0.35] (6.19,1.62) circle (1.7pt);
\draw[blue!60!black,line width=0.4pt,fill=blue!60!black,fill opacity=0.35] (6.44,1.93) circle (1.7pt);
\draw[blue!60!black,line width=0.4pt,fill=blue!60!black,fill opacity=0.35] (6.46,1.89) circle (1.7pt);
\draw[blue!60!black,line width=0.4pt,fill=blue!60!black,fill opacity=0.35] (6.57,2.01) circle (1.7pt);
\draw[blue!60!black,line width=0.4pt,fill=blue!60!black,fill opacity=0.35] (6.58,2.12) circle (1.7pt);
\draw[blue!60!black,line width=0.4pt,fill=blue!60!black,fill opacity=0.35] (7.09,2.43) circle (1.7pt);
\draw[blue!60!black,line width=0.4pt,fill=blue!60!black,fill opacity=0.35] (7.09,2.35) circle (1.7pt);
\draw[blue!60!black,line width=0.4pt,fill=blue!60!black,fill opacity=0.35] (7.29,2.35) circle (1.7pt);
\draw[blue!60!black,line width=0.4pt,fill=blue!60!black,fill opacity=0.35] (7.21,2.43) circle (1.7pt);
\draw[orange!85!black,line width=0.4pt,fill=orange!85!black,fill opacity=0.35] (6.12,1.09) circle (1.7pt);
\draw[orange!85!black,line width=0.4pt,fill=orange!85!black,fill opacity=0.35] (6.17,1.20) circle (1.7pt);
\draw[orange!85!black,line width=0.4pt,fill=orange!85!black,fill opacity=0.35] (6.22,1.51) circle (1.7pt);
\draw[orange!85!black,line width=0.4pt,fill=orange!85!black,fill opacity=0.35] (6.19,1.74) circle (1.7pt);
\draw[orange!85!black,line width=0.4pt,fill=orange!85!black,fill opacity=0.35] (6.43,1.82) circle (1.7pt);
\draw[orange!85!black,line width=0.4pt,fill=orange!85!black,fill opacity=0.35] (6.43,1.59) circle (1.7pt);
\draw[orange!85!black,line width=0.4pt,fill=orange!85!black,fill opacity=0.35] (6.72,1.70) circle (1.7pt);
\draw[orange!85!black,line width=0.4pt,fill=orange!85!black,fill opacity=0.35] (6.57,2.12) circle (1.7pt);
\draw[orange!85!black,line width=0.4pt,fill=orange!85!black,fill opacity=0.35] (7.20,1.16) circle (1.7pt);
\draw[orange!85!black,line width=0.4pt,fill=orange!85!black,fill opacity=0.35] (7.17,1.28) circle (1.7pt);
\draw[orange!85!black,line width=0.4pt,fill=orange!85!black,fill opacity=0.35] (7.27,2.24) circle (1.7pt);
\draw[orange!85!black,line width=0.4pt,fill=orange!85!black,fill opacity=0.35] (7.05,2.39) circle (1.7pt);
\draw[orange!85!black,line width=0.4pt,fill=orange!85!black,fill opacity=0.35] (6.15,1.09) circle (1.7pt);
\draw[orange!85!black,line width=0.4pt,fill=orange!85!black,fill opacity=0.35] (6.14,1.09) circle (1.7pt);
\draw[orange!85!black,line width=0.4pt,fill=orange!85!black,fill opacity=0.35] (6.24,1.20) circle (1.7pt);
\draw[orange!85!black,line width=0.4pt,fill=orange!85!black,fill opacity=0.35] (6.17,1.62) circle (1.7pt);
\draw[orange!85!black,line width=0.4pt,fill=orange!85!black,fill opacity=0.35] (6.39,1.78) circle (1.7pt);
\draw[orange!85!black,line width=0.4pt,fill=orange!85!black,fill opacity=0.35] (6.41,1.47) circle (1.7pt);
\draw[orange!85!black,line width=0.4pt,fill=orange!85!black,fill opacity=0.35] (6.57,1.74) circle (1.7pt);
\draw[orange!85!black,line width=0.4pt,fill=orange!85!black,fill opacity=0.35] (6.44,1.78) circle (1.7pt);
\draw[orange!85!black,line width=0.4pt,fill=orange!85!black,fill opacity=0.35] (8.89,0.74) circle (1.7pt);
\draw[orange!85!black,line width=0.4pt,fill=orange!85!black,fill opacity=0.35] (8.84,0.74) circle (1.7pt);
\draw[orange!85!black,line width=0.4pt,fill=orange!85!black,fill opacity=0.35] (7.26,1.97) circle (1.7pt);
\draw[orange!85!black,line width=0.4pt,fill=orange!85!black,fill opacity=0.35] (7.24,2.47) circle (1.7pt);
\draw[red!75!black,line width=0.4pt,fill=red!75!black,fill opacity=0.35] (6.13,0.86) circle (1.7pt);
\draw[red!75!black,line width=0.4pt,fill=red!75!black,fill opacity=0.35] (6.12,0.82) circle (1.7pt);
\draw[red!75!black,line width=0.4pt,fill=red!75!black,fill opacity=0.35] (6.16,0.82) circle (1.7pt);
\draw[red!75!black,line width=0.4pt,fill=red!75!black,fill opacity=0.35] (6.15,1.09) circle (1.7pt);
\draw[red!75!black,line width=0.4pt,fill=red!75!black,fill opacity=0.35] (6.24,1.28) circle (1.7pt);
\draw[red!75!black,line width=0.4pt,fill=red!75!black,fill opacity=0.35] (6.29,1.20) circle (1.7pt);
\draw[red!75!black,line width=0.4pt,fill=red!75!black,fill opacity=0.35] (6.49,1.32) circle (1.7pt);
\draw[red!75!black,line width=0.4pt,fill=red!75!black,fill opacity=0.35] (6.52,1.36) circle (1.7pt);
\draw[red!75!black,line width=0.4pt,fill=red!75!black,fill opacity=0.35] (8.87,0.74) circle (1.7pt);
\draw[red!75!black,line width=0.4pt,fill=red!75!black,fill opacity=0.35] (8.60,0.74) circle (1.7pt);
\draw[red!75!black,line width=0.4pt,fill=red!75!black,fill opacity=0.35] (7.15,0.82) circle (1.7pt);
\draw[red!75!black,line width=0.4pt,fill=red!75!black,fill opacity=0.35] (7.24,0.82) circle (1.7pt);
\draw[red!75!black,line width=0.4pt,fill=red!75!black,fill opacity=0.35] (8.41,0.86) circle (1.7pt);
\draw[red!75!black,line width=0.4pt,fill=red!75!black,fill opacity=0.35] (7.85,0.70) circle (1.7pt);
\draw[red!75!black,line width=0.4pt,fill=red!75!black,fill opacity=0.35] (8.33,0.74) circle (1.7pt);
\draw[red!75!black,line width=0.4pt,fill=red!75!black,fill opacity=0.35] (8.80,0.74) circle (1.7pt);
\draw[red!75!black,line width=0.4pt,fill=red!75!black,fill opacity=0.35] (8.90,0.74) circle (1.7pt);
\draw[red!75!black,line width=0.4pt,fill=red!75!black,fill opacity=0.35] (8.90,0.74) circle (1.7pt);
\draw[red!75!black,line width=0.4pt,fill=red!75!black,fill opacity=0.35] (8.62,0.74) circle (1.7pt);
\draw[red!75!black,line width=0.4pt,fill=red!75!black,fill opacity=0.35] (8.90,0.74) circle (1.7pt);
\draw[red!75!black,line width=0.4pt,fill=red!75!black,fill opacity=0.35] (5.85,0.74) circle (1.7pt);
\draw[red!75!black,line width=0.4pt,fill=red!75!black,fill opacity=0.35] (5.90,0.74) circle (1.7pt);
\draw[red!75!black,line width=0.4pt,fill=red!75!black,fill opacity=0.35] (8.90,0.74) circle (1.7pt);
\draw[red!75!black,line width=0.4pt,fill=red!75!black,fill opacity=0.35] (8.90,0.74) circle (1.7pt);
\draw[black!55,line width=0.4pt] (10.15,0.55) rectangle (13.45,2.70);
\node[anchor=south,font=\footnotesize] at (11.80,2.80) {(c) ViT--DomainNet};
\draw[black!55] (10.15,0.55) -- (10.15,0.46);
\node[anchor=north,font=\scriptsize,black!70] at (10.15,0.44) {0};
\draw[black!55] (11.80,0.55) -- (11.80,0.46);
\node[anchor=north,font=\scriptsize,black!70] at (11.80,0.44) {0.5};
\draw[black!55] (13.45,0.55) -- (13.45,0.46);
\node[anchor=north,font=\scriptsize,black!70] at (13.45,0.44) {1};
\draw[black!12] (10.15,0.55) -- (13.45,0.55);
\draw[black!55] (10.15,0.55) -- (10.08,0.55);
\node[anchor=east,font=\scriptsize,black!70] at (10.06,0.55) {0.00};
\draw[black!12] (10.15,1.62) -- (13.45,1.62);
\draw[black!55] (10.15,1.62) -- (10.08,1.62);
\node[anchor=east,font=\scriptsize,black!70] at (10.06,1.62) {0.38};
\draw[black!12] (10.15,2.70) -- (13.45,2.70);
\draw[black!55] (10.15,2.70) -- (10.08,2.70);
\node[anchor=east,font=\scriptsize,black!70] at (10.06,2.70) {0.76};
\draw[red!60,dashed,line width=0.5pt] (10.15,0.61) -- (13.45,0.61);
\node[anchor=south east,font=\scriptsize,red!60] at (13.40,0.64) {chance $1/50$};
\draw[blue!60!black,line width=0.4pt,fill=blue!60!black,fill opacity=0.35] (12.38,1.69) circle (1.7pt);
\draw[blue!60!black,line width=0.4pt,fill=blue!60!black,fill opacity=0.35] (12.39,1.70) circle (1.7pt);
\draw[blue!60!black,line width=0.4pt,fill=blue!60!black,fill opacity=0.35] (12.43,1.85) circle (1.7pt);
\draw[blue!60!black,line width=0.4pt,fill=blue!60!black,fill opacity=0.35] (12.44,1.86) circle (1.7pt);
\draw[blue!60!black,line width=0.4pt,fill=blue!60!black,fill opacity=0.35] (12.51,1.94) circle (1.7pt);
\draw[blue!60!black,line width=0.4pt,fill=blue!60!black,fill opacity=0.35] (12.51,1.94) circle (1.7pt);
\draw[blue!60!black,line width=0.4pt,fill=blue!60!black,fill opacity=0.35] (12.59,2.10) circle (1.7pt);
\draw[blue!60!black,line width=0.4pt,fill=blue!60!black,fill opacity=0.35] (12.59,2.11) circle (1.7pt);
\draw[blue!60!black,line width=0.4pt,fill=blue!60!black,fill opacity=0.35] (12.63,2.45) circle (1.7pt);
\draw[blue!60!black,line width=0.4pt,fill=blue!60!black,fill opacity=0.35] (12.63,2.45) circle (1.7pt);
\draw[blue!60!black,line width=0.4pt,fill=blue!60!black,fill opacity=0.35] (12.63,2.46) circle (1.7pt);
\draw[blue!60!black,line width=0.4pt,fill=blue!60!black,fill opacity=0.35] (12.63,2.46) circle (1.7pt);
\draw[blue!60!black,line width=0.4pt,fill=blue!60!black,fill opacity=0.35] (12.37,1.69) circle (1.7pt);
\draw[blue!60!black,line width=0.4pt,fill=blue!60!black,fill opacity=0.35] (12.38,1.70) circle (1.7pt);
\draw[blue!60!black,line width=0.4pt,fill=blue!60!black,fill opacity=0.35] (12.43,1.85) circle (1.7pt);
\draw[blue!60!black,line width=0.4pt,fill=blue!60!black,fill opacity=0.35] (12.43,1.87) circle (1.7pt);
\draw[blue!60!black,line width=0.4pt,fill=blue!60!black,fill opacity=0.35] (12.51,1.93) circle (1.7pt);
\draw[blue!60!black,line width=0.4pt,fill=blue!60!black,fill opacity=0.35] (12.52,1.94) circle (1.7pt);
\draw[blue!60!black,line width=0.4pt,fill=blue!60!black,fill opacity=0.35] (12.59,2.10) circle (1.7pt);
\draw[blue!60!black,line width=0.4pt,fill=blue!60!black,fill opacity=0.35] (12.59,2.11) circle (1.7pt);
\draw[blue!60!black,line width=0.4pt,fill=blue!60!black,fill opacity=0.35] (12.56,2.44) circle (1.7pt);
\draw[blue!60!black,line width=0.4pt,fill=blue!60!black,fill opacity=0.35] (12.56,2.44) circle (1.7pt);
\draw[blue!60!black,line width=0.4pt,fill=blue!60!black,fill opacity=0.35] (12.64,2.46) circle (1.7pt);
\draw[blue!60!black,line width=0.4pt,fill=blue!60!black,fill opacity=0.35] (12.64,2.46) circle (1.7pt);
\draw[orange!85!black,line width=0.4pt,fill=orange!85!black,fill opacity=0.35] (12.36,1.68) circle (1.7pt);
\draw[orange!85!black,line width=0.4pt,fill=orange!85!black,fill opacity=0.35] (12.37,1.69) circle (1.7pt);
\draw[orange!85!black,line width=0.4pt,fill=orange!85!black,fill opacity=0.35] (12.42,1.83) circle (1.7pt);
\draw[orange!85!black,line width=0.4pt,fill=orange!85!black,fill opacity=0.35] (12.43,1.83) circle (1.7pt);
\draw[orange!85!black,line width=0.4pt,fill=orange!85!black,fill opacity=0.35] (12.49,1.92) circle (1.7pt);
\draw[orange!85!black,line width=0.4pt,fill=orange!85!black,fill opacity=0.35] (12.50,1.93) circle (1.7pt);
\draw[orange!85!black,line width=0.4pt,fill=orange!85!black,fill opacity=0.35] (12.57,2.09) circle (1.7pt);
\draw[orange!85!black,line width=0.4pt,fill=orange!85!black,fill opacity=0.35] (12.57,2.09) circle (1.7pt);
\draw[orange!85!black,line width=0.4pt,fill=orange!85!black,fill opacity=0.35] (12.03,2.06) circle (1.7pt);
\draw[orange!85!black,line width=0.4pt,fill=orange!85!black,fill opacity=0.35] (12.03,2.06) circle (1.7pt);
\draw[orange!85!black,line width=0.4pt,fill=orange!85!black,fill opacity=0.35] (12.64,2.47) circle (1.7pt);
\draw[orange!85!black,line width=0.4pt,fill=orange!85!black,fill opacity=0.35] (12.64,2.47) circle (1.7pt);
\draw[orange!85!black,line width=0.4pt,fill=orange!85!black,fill opacity=0.35] (12.31,1.67) circle (1.7pt);
\draw[orange!85!black,line width=0.4pt,fill=orange!85!black,fill opacity=0.35] (12.32,1.69) circle (1.7pt);
\draw[orange!85!black,line width=0.4pt,fill=orange!85!black,fill opacity=0.35] (12.44,1.77) circle (1.7pt);
\draw[orange!85!black,line width=0.4pt,fill=orange!85!black,fill opacity=0.35] (12.42,1.79) circle (1.7pt);
\draw[orange!85!black,line width=0.4pt,fill=orange!85!black,fill opacity=0.35] (12.47,1.89) circle (1.7pt);
\draw[orange!85!black,line width=0.4pt,fill=orange!85!black,fill opacity=0.35] (12.46,1.90) circle (1.7pt);
\draw[orange!85!black,line width=0.4pt,fill=orange!85!black,fill opacity=0.35] (12.55,2.05) circle (1.7pt);
\draw[orange!85!black,line width=0.4pt,fill=orange!85!black,fill opacity=0.35] (12.55,2.06) circle (1.7pt);
\draw[orange!85!black,line width=0.4pt,fill=orange!85!black,fill opacity=0.35] (11.21,0.66) circle (1.7pt);
\draw[orange!85!black,line width=0.4pt,fill=orange!85!black,fill opacity=0.35] (11.14,0.65) circle (1.7pt);
\draw[orange!85!black,line width=0.4pt,fill=orange!85!black,fill opacity=0.35] (12.53,2.38) circle (1.7pt);
\draw[orange!85!black,line width=0.4pt,fill=orange!85!black,fill opacity=0.35] (12.54,2.39) circle (1.7pt);
\draw[red!75!black,line width=0.4pt,fill=red!75!black,fill opacity=0.35] (12.07,1.42) circle (1.7pt);
\draw[red!75!black,line width=0.4pt,fill=red!75!black,fill opacity=0.35] (12.04,1.47) circle (1.7pt);
\draw[red!75!black,line width=0.4pt,fill=red!75!black,fill opacity=0.35] (12.18,1.55) circle (1.7pt);
\draw[red!75!black,line width=0.4pt,fill=red!75!black,fill opacity=0.35] (12.14,1.59) circle (1.7pt);
\draw[red!75!black,line width=0.4pt,fill=red!75!black,fill opacity=0.35] (12.30,1.55) circle (1.7pt);
\draw[red!75!black,line width=0.4pt,fill=red!75!black,fill opacity=0.35] (12.24,1.60) circle (1.7pt);
\draw[red!75!black,line width=0.4pt,fill=red!75!black,fill opacity=0.35] (12.27,1.75) circle (1.7pt);
\draw[red!75!black,line width=0.4pt,fill=red!75!black,fill opacity=0.35] (12.25,1.79) circle (1.7pt);
\draw[red!75!black,line width=0.4pt,fill=red!75!black,fill opacity=0.35] (10.71,0.61) circle (1.7pt);
\draw[red!75!black,line width=0.4pt,fill=red!75!black,fill opacity=0.35] (10.53,0.60) circle (1.7pt);
\draw[red!75!black,line width=0.4pt,fill=red!75!black,fill opacity=0.35] (11.76,1.68) circle (1.7pt);
\draw[red!75!black,line width=0.4pt,fill=red!75!black,fill opacity=0.35] (11.86,1.86) circle (1.7pt);
\draw[red!75!black,line width=0.4pt,fill=red!75!black,fill opacity=0.35] (11.13,0.61) circle (1.7pt);
\draw[red!75!black,line width=0.4pt,fill=red!75!black,fill opacity=0.35] (10.86,0.60) circle (1.7pt);
\draw[red!75!black,line width=0.4pt,fill=red!75!black,fill opacity=0.35] (11.06,0.59) circle (1.7pt);
\draw[red!75!black,line width=0.4pt,fill=red!75!black,fill opacity=0.35] (10.94,0.63) circle (1.7pt);
\draw[red!75!black,line width=0.4pt,fill=red!75!black,fill opacity=0.35] (11.42,0.60) circle (1.7pt);
\draw[red!75!black,line width=0.4pt,fill=red!75!black,fill opacity=0.35] (11.06,0.61) circle (1.7pt);
\draw[red!75!black,line width=0.4pt,fill=red!75!black,fill opacity=0.35] (11.40,0.60) circle (1.7pt);
\draw[red!75!black,line width=0.4pt,fill=red!75!black,fill opacity=0.35] (10.95,0.58) circle (1.7pt);
\draw[red!75!black,line width=0.4pt,fill=red!75!black,fill opacity=0.35] (10.73,0.61) circle (1.7pt);
\draw[red!75!black,line width=0.4pt,fill=red!75!black,fill opacity=0.35] (10.43,0.63) circle (1.7pt);
\draw[red!75!black,line width=0.4pt,fill=red!75!black,fill opacity=0.35] (10.92,0.57) circle (1.7pt);
\draw[red!75!black,line width=0.4pt,fill=red!75!black,fill opacity=0.35] (11.01,0.58) circle (1.7pt);
\node[anchor=north,font=\footnotesize] at (10.30,0.13) {mean confidence on the unlabeled pool};
\node[rotate=90,anchor=south,font=\footnotesize] at (4.72,1.63) {target accuracy};
\draw[red!75!black,line width=0.4pt,fill=red!75!black,fill opacity=0.35] (6.60,-0.37) circle (1.7pt);
\node[anchor=west,font=\scriptsize] at (6.74,-0.37) {2--3 bits};
\draw[orange!85!black,line width=0.4pt,fill=orange!85!black,fill opacity=0.35] (8.55,-0.37) circle (1.7pt);
\node[anchor=west,font=\scriptsize] at (8.69,-0.37) {4--5 bits};
\draw[blue!60!black,line width=0.4pt,fill=blue!60!black,fill opacity=0.35] (10.50,-0.37) circle (1.7pt);
\node[anchor=west,font=\scriptsize] at (10.64,-0.37) {6--8 bits};
\end{tikzpicture}
\caption{Why the label-free rule needs help, and why confidence cannot supply it.
(a) Two candidates on a level set of the teacher distortion, $g_D$ and $g_A$, are
indistinguishable to the label-free rule; the supervised term breaks the tie, and
alignment prefers movement along $e_y-p_f$, while an overconcentrated candidate sits
near a vertex, confident and far from the teacher. (b, c) Each point is one of the
$72$ candidates of a single run, positioned by its mean confidence on the unlabeled
pool and its target accuracy, coloured by bit width, with the dashed line at chance.
On ResNet50 confidence and accuracy are negatively related across the family
($\rho=-0.32$ in this run); on ViT--DomainNet the relation is strongly positive
($\rho=+0.93$), which is why the same estimators select well there.}
\label{fig:collapse}
\end{figure*}

The output-distribution group is a different matter, and the multiple-comparison
correction sharpens it. On ResNet50 we cannot distinguish any of the four from the anchor with the available
clustered sample, and all four incur substantial mean regret. On MobileNetV2 three of them have clearly lower mean regret than the anchor. Teacher distortion is therefore not the strongest label-free rule available, and we do
not claim it is; what it
offers is stability under the degradation regime that defeats the confidence group,
which makes it usable as a shrinkage target rather than as a selector.

Comparing the two blocks of Table~\ref{tab:sota} answers whether labels are worth
acquiring, and the answer differs sharply by architecture. On ResNet50 anchoring reduces mean regret from $0.166$ to $0.007$ as the budget grows,
against $0.212$ for the best label-free rule. On MobileNetV2
a label-free rule wins at every budget up to the complete pool: COT attains $0.007$
out of a cross-entropy spread of $11.4$, and no supervised selector improves on it
until $n=300$. The contrast is consistent with the headroom the best label-free rule
leaves, which differs by a factor of twenty between the two, $1.4\%$ of the
attainable damage against $0.065\%$. Labels buy selection accuracy where the label-free problem is unsolved, a property of the candidate family, not of the selector.

Distortion is immune to that failure yet still leaves regret on the table. A post-hoc
audit (Table~\ref{tab:ece}) shows that on ResNet50 the candidate chosen by validation
cross-entropy has lower calibration error and cross-entropy but lower accuracy:
paired run-level differences are $-0.122$ $[-0.161, -0.082]$ in ECE and
$-0.208$ $[-0.310, -0.113]$ in cross-entropy. Calibration, not accuracy, is the dominant source of
disagreement, and this is what a $0$--$1$ loss cannot recover.
Supplement~\ref{supp:diagnostics} checks that the difficulty is not an artifact of one
catastrophic candidate by restricting the family.

\begin{table}[t]
\centering
\scriptsize
\setlength{\tabcolsep}{2pt}
\caption{Post-hoc calibration audit, computed on the first execution of each CNN
setting. Upper block: medians over its $15$ runs per
architecture, ECE with $15$ equal-width bins, accuracy top-$1$. Lower block: mean
run-level differences (validation CE minus distortion) with run-bootstrap $95\%$
intervals, leading zeros omitted. Test labels are used only for explanation.}
\label{tab:ece}
\begin{tabular}{@{}llrrr@{}}
\toprule
Setting & Selected by & ECE & Acc. & CE \\
\midrule
ResNet50 & Distortion & 0.216 & 0.332 & 2.865 \\
         & Validation CE & 0.083 & 0.327 & 2.745 \\
MobileNetV2 & Distortion & 0.227 & 0.308 & 3.019 \\
         & Validation CE & 0.195 & 0.308 & 2.884 \\
\midrule
\multicolumn{5}{@{}l}{\emph{Paired difference, validation CE minus distortion}} \\
\multicolumn{2}{@{}l}{ResNet50} & $-.122$ & $-.042$ & $-.208$ \\
\multicolumn{2}{@{}l}{\quad $95\%$ CI} & $[-.161,-.082]$ & $[-.056,-.027]$ & $[-.310,-.113]$ \\
\multicolumn{2}{@{}l}{MobileNetV2} & $-.013$ & $+.006$ & $-.084$ \\
\multicolumn{2}{@{}l}{\quad $95\%$ CI} & $[-.019,-.007]$ & $[+.002,+.009]$ & $[-.121,-.049]$ \\
\bottomrule
\end{tabular}
\end{table}

\subsection{Selection with Scarce Clean Labels}
\label{subsec:budgets}
Spending a small number of target labels changes the picture.
Table~\ref{tab:main} reports median regret across budgets for the three selectors.

\begin{table}[t]
\centering\scriptsize\setlength{\tabcolsep}{3pt}
\caption{Median target cross-entropy selection regret with clean labels, over
run--labeled-subset outcomes. Dist.\ is label-free and constant across budgets.
Boldface marks the lowest median among the supervised selectors in each row. The largest budget is the largest evaluated in each setting,
not the size of its labeled pool. Runs are $45$ per CNN architecture, pooling three
end-to-end executions, and $17$ and $27$ on the two ViT settings.}
\label{tab:main}
\begin{tabular}{@{}llcccc@{}}
\toprule
Setting & $n$ & Dist.\ & Val-CE & CE-combo & Align \\
\midrule
ResNet50--DomainNet & 10 & 0.2323 & \textbf{0.1756} & 0.1941 & 0.1870 \\
 & 25 & 0.2323 & 0.0972 & \textbf{0.0948} & 0.1143 \\
 & 50 & 0.2323 & 0.0571 & \textbf{0.0498} & 0.0635 \\
 & 100 & 0.2323 & \textbf{0.0267} & 0.0271 & 0.0342 \\
 & 150 & 0.2323 & \textbf{0.0232} & 0.0239 & 0.0270 \\
 & 300 & 0.2323 & \textbf{0.0000} & 0.0012 & 0.0024 \\
\midrule
MobileNetV2--DomainNet & 10 & 0.0778 & 0.2479 & 0.1055 & \textbf{0.0816} \\
 & 25 & 0.0778 & 0.0644 & \textbf{0.0349} & 0.0399 \\
 & 50 & 0.0778 & \textbf{0.0194} & 0.0221 & 0.0241 \\
 & 100 & 0.0778 & 0.0073 & \textbf{0.0046} & 0.0049 \\
 & 150 & 0.0778 & \textbf{0.0024} & 0.0033 & 0.0026 \\
 & 300 & 0.0778 & \textbf{0.0000} & 0.0000 & 0.0000 \\
\midrule
ViT--DomainNet & 10 & 0.0200 & 0.1722 & 0.0725 & \textbf{0.0501} \\
 & 25 & 0.0200 & 0.0746 & 0.0445 & \textbf{0.0434} \\
 & 50 & 0.0200 & 0.0482 & \textbf{0.0207} & 0.0295 \\
 & 100 & 0.0200 & 0.0188 & \textbf{0.0092} & 0.0183 \\
 & 150 & 0.0200 & 0.0151 & \textbf{0.0082} & 0.0110 \\
\midrule
ViT--CIFAR-20 & 10 & 0.1009 & 0.1938 & \textbf{0.0990} & 0.1325 \\
 & 25 & 0.1009 & 0.0952 & \textbf{0.0856} & 0.1010 \\
 & 50 & 0.1009 & \textbf{0.0436} & 0.0638 & 0.0764 \\
 & 100 & 0.1009 & \textbf{0.0352} & 0.0483 & 0.0715 \\
 & 150 & 0.1009 & \textbf{0.0263} & 0.0421 & 0.0531 \\
 & 200 & 0.1009 & \textbf{0.0219} & 0.0382 & 0.0458 \\
\bottomrule
\end{tabular}
\end{table}



With ten labels CE-combo improves on direct validation in \emph{all four} settings. Across the twenty-three comparisons of the budget ablation
(Supplement Table~\ref{tab:sweep}) it wins in eleven and loses in twelve, strongly
budget-dependent: four of four at ten labels, three of four at twenty-five, and four of
fifteen from fifty onward. At ten labels the mean regret falls from $0.227$ to $0.103$ on
MobileNetV2, from $0.222$ to $0.116$ on ViT--CIFAR-20, from $0.208$ to $0.140$ on
ViT--DomainNet and from $0.182$ to $0.166$ on ResNet50. On ResNet50 the
median does not fall, so the gain there comes from removing severe failures rather
than from shifting the typical outcome. From fifty labels onward it loses in eleven of fifteen comparisons, so its stabilizing
benefit usually stops compensating for anchor bias once the supervised estimate is
precise. The value of anchoring is a small-budget effect.

Anchoring also protects the upper tail: at ten labels the $95$th percentile of regret falls from $0.58$ to $0.49$ on
ResNet50 and from $1.18$ to $0.71$ on MobileNetV2, and $\\Pr(R>0.1)$ from
$0.30$ to $0.23$ on the latter (Supplement Table~\ref{tab:tail}). Paired run-level inference for both
supervision regimes is in Table~\ref{tab:paired}.



\subsection{What the Directional Term Contributes}
\label{subsec:direction}
The directional term admits a control that isolates exactly what it claims to add.
Permuting the observed labels before the statistic is formed leaves the marginal
distribution and scale of $A_{\mathcal S}$ intact and destroys only the correspondence
between labels and inputs, so the control keeps the shrinkage of \eqref{eq:family} and
none of its direction. It receives the identical coefficient grid and cross-validation.

Supplement Table~\ref{tab:control} reports it. The control behaves as its construction
predicts, though not exactly: its regret tracks pure distortion to within $0.020$ across settings and budgets, with a
mean absolute deviation of $0.006$. The coefficient it selects is zero in $71\%$ of ResNet50 and $76\%$ of MobileNetV2 runs,
median zero in both, so with the correspondence destroyed the procedure usually reverts to the anchor, and
the residual deviation comes from the runs where it does not.
That makes the comparison informative only where reverting to the anchor is
costly.



ResNet50 is that case, and there the direction contributes throughout: alignment beats the control at every budget, from $0.1870$ against $0.2170$ at ten
labels to $0.0270$ against $0.2320$ at $150$, the gap widening as labels accumulate. In the other three settings distortion is already within a few hundredths of the
optimum, so reverting to it forfeits little and the control is not separated from
alignment. The directional term therefore earns its place
conditionally. Relating the gain from anchoring to the regret of pure distortion across the nineteen setting--budget cells beyond the smallest budget shows the condition is not a property of the anchor: the
association is weak and of the wrong sign (Spearman $-0.25$), which we read as a descriptive diagnostic, not a formal test, since the cells are not
independent. What governs
the gain is the label budget. With ten labels anchoring improves on direct
validation in all four settings, and from fifty labels onward it loses in three of them. That is what a shrinkage rule should do: once the supervised estimate is reliable, the
anchor stops helping.

The two components behave heterogeneously across settings rather than along one axis. Retaining only the teacher component, whose statistic uses no
observed class, is the best of the three variants on ResNet50 at ten labels ($0.145$) and the worst on the other three settings at that budget. Alignment is thus not a mechanism of uniform value. It is a correction whose value varies across settings, and anchor regret does not
predict that variation; the label budget is the clearer moderator.

\subsection{Selection under Symmetric Label Corruption}
\label{subsec:noise}
Each sampled label is retained with probability $1-\eta$ and otherwise replaced
uniformly by one of the $K-1$ incorrect classes, so a corrupted label always differs from
the clean one. Corruption changes only the labels supplied to the selectors, never the candidates or
test set.

Table~\ref{tab:noise} shows that anchoring becomes more valuable as supervision
degrades. At $\eta=0.4$ both anchored selectors improve on direct validation on both
architectures. The directional variant attains the lowest median in each case, reducing
regret from $0.165$ to $0.102$ on ResNet50 and from $0.779$ to $0.393$ on
MobileNetV2.



\subsection{Scarce and Corrupted Labels Together}
\label{subsec:grid}
The experiments so far vary one factor at a time, leaving out the regime a practitioner
is most likely to face. Table~\ref{tab:grid} crosses the two, giving every selector the
same $n$ labels at the same rate and corrupting those its coefficient selection consumes
as well.

\begin{table*}[t]
\centering
\scriptsize
\setlength{\tabcolsep}{4.5pt}
\caption{Label budget crossed with corruption rate: mean target cross-entropy regret over $15$ runs per
architecture and ten corruption realizations per cell. Unlike the clean-label tables
this grid is computed on the first complete execution of each setting, since crossing
five budgets with three corruption rates and ten realizations makes pooling all three
executions prohibitively slow.
Every selector receives the same $n$ labels at the same rate $\eta$, and the labels
its coefficient selection consumes are corrupted too. Boldface marks the lowest mean
in each cell.}
\label{tab:grid}
\begin{tabular}{@{}ll ccc ccc ccc@{}}
\toprule
& & \multicolumn{3}{c}{$\eta=0$} & \multicolumn{3}{c}{$\eta=0.2$} & \multicolumn{3}{c}{$\eta=0.4$} \\
\cmidrule(lr){3-5}\cmidrule(lr){6-8}\cmidrule(lr){9-11}
Setting & $n$ & Val-CE & Val-Acc & CE-combo & Val-CE & Val-Acc & CE-combo & Val-CE & Val-Acc & CE-combo \\
\midrule
ResNet50--DomainNet & 10 & 0.168 & 0.630 & \textbf{0.159} & 0.326 & 1.497 & \textbf{0.234} & 0.393 & 2.485 & \textbf{0.278} \\
 & 25 & 0.100 & 0.249 & \textbf{0.091} & 0.145 & 0.427 & \textbf{0.111} & 0.361 & 0.792 & \textbf{0.230} \\
 & 50 & \textbf{0.051} & 0.262 & 0.053 & 0.118 & 0.269 & \textbf{0.083} & 0.317 & 0.262 & \textbf{0.216} \\
 & 100 & \textbf{0.025} & 0.227 & 0.030 & 0.104 & 0.239 & \textbf{0.071} & 0.288 & 0.240 & \textbf{0.194} \\
 & 300 & \textbf{0.008} & 0.219 & 0.010 & 0.087 & 0.231 & \textbf{0.073} & 0.232 & 0.240 & \textbf{0.202} \\
\midrule
MobileNetV2--DomainNet & 10 & 0.285 & 1.140 & \textbf{0.127} & 0.728 & 2.021 & \textbf{0.406} & 0.864 & 2.219 & \textbf{0.485} \\
 & 25 & 0.107 & 0.201 & \textbf{0.046} & 0.308 & 0.464 & \textbf{0.143} & 0.708 & 0.851 & \textbf{0.513} \\
 & 50 & 0.048 & 0.054 & \textbf{0.027} & 0.285 & \textbf{0.079} & 0.135 & 0.765 & \textbf{0.230} & 0.601 \\
 & 100 & 0.015 & 0.044 & \textbf{0.008} & 0.260 & \textbf{0.037} & 0.136 & 0.768 & \textbf{0.036} & 0.684 \\
 & 300 & \textbf{0.001} & 0.018 & 0.001 & 0.193 & \textbf{0.019} & 0.099 & 0.844 & \textbf{0.028} & 0.803 \\
\bottomrule
\end{tabular}
\end{table*}

Supplement~\ref{supp:anchorvar} shows that neither the input set on which the anchor is
evaluated nor its identity as the teacher is essential to what follows.
Anchoring helps far more here than with clean labels. Across the twenty cells with
$\eta>0$ CE-combo improves on direct validation in every one, against eleven of twenty-three in the clean sweep. The margins are large where both scarcity and
corruption bite: at ten labels and $\eta=0.4$ mean regret falls from $0.393$ to
$0.278$ on ResNet50 and from $0.864$ to $0.485$ on MobileNetV2. Corruption inflates the variance of the supervised estimate,
and that is precisely when shrinkage should pay. Here it pays at every budget, and not merely the smallest.

Validation accuracy behaves as the attenuation identity predicts and thereby marks
the limit of the anchored rule. At ten labels it is catastrophic, above $1.1$ on MobileNetV2, since a $30$-class
accuracy estimate from ten examples is almost uninformative. But its ordering degrades only by a common factor, so once the
estimate is precise it becomes the most corruption-tolerant of the three: on MobileNetV2 it attains the lowest regret in all six cells with
$n\ge 50$ and $\eta>0$, reaching $0.028$ at $n=300$ and $\eta=0.4$ where CE-combo
reaches $0.803$. This points to a division of labour: anchor when labels are few, and consider a bounded
accuracy criterion when they are plentiful but unreliable, a pattern we observe on
MobileNetV2 and not on ResNet50.

\subsection{Does Selection Survive a Cost Constraint?}
\label{subsec:cost}
Minimizing predictive cross-entropy over the whole family answers a question a
practitioner would not ask alone, since candidates differ in deployment cost. We repeat
the exercise as a constrained problem, $\min_g L(g)$
subject to $C(g)\le B$, with $C$ the weight memory relative to full precision, with the frontier in
Supplement Table~\ref{tab:cost}.

Supplement Table~\ref{tab:costsel} applies the selectors inside the constraint, with the
candidate set, the labeled sample, the coefficient search and the oracle all
restricted to $\mathcal G_B$. The small-budget pattern survives: at ten labels
CE-combo improves on direct validation in seven of the eight
budget--architecture cells and loses in one, the tightest budget on MobileNetV2,
where the admissible family is small enough that the choice barely matters. Restricting the problem neither removes the advantage of anchoring at small budgets nor
creates one at large budgets.

Two effects separate. The price of the constraint is strongly architecture-dependent: on ResNet50 the frontier
is nearly flat above a fifth of full precision and $\le 22\%$ costs $0.007$, while on
MobileNetV2 it is steep everywhere and $\le 22\%$ already costs $0.868$. The effect that matters for selection survives intact: within every budget the spread
between the worst and best admissible candidate stays above $10$ in cross-entropy,
essentially unchanged. Constraining cost removes the most expensive candidates but
does not make the remainder interchangeable. The problem we study therefore survives a budget instead of being dissolved by it.

\begin{table}[t]
\centering
\scriptsize
\setlength{\tabcolsep}{4pt}
\caption{Empirical validation of the symmetric-corruption attenuation identity at the
complete labeled pool, pooling all three executions, $45$ runs per architecture and
ten corruption realizations per rate. Entries are through-origin slopes relating the
corrupted quantity to its clean counterpart. The teacher component is unchanged to
numerical precision at every rate.}
\label{tab:attenuation}
\begin{tabular}{@{}llccc@{}}
\toprule
Setting & $\eta$ & Predicted & Label comp.\ & Val.\ accuracy \\
\midrule
ResNet50--DomainNet & 10\% & 0.897 & 0.881 & 0.881 \\
 & 20\% & 0.793 & 0.800 & 0.802 \\
 & 30\% & 0.690 & 0.683 & 0.687 \\
 & 40\% & 0.586 & 0.577 & 0.579 \\
\midrule
MobileNetV2--DomainNet & 10\% & 0.897 & 0.882 & 0.882 \\
 & 20\% & 0.793 & 0.803 & 0.802 \\
 & 30\% & 0.690 & 0.688 & 0.689 \\
 & 40\% & 0.586 & 0.575 & 0.575 \\
\bottomrule
\end{tabular}
\end{table}

\subsection{Does the Theory Predict the Measurement?}
\label{subsec:attenuation}
The attenuation result makes a numerical prediction. At the largest budget and
$\eta\in\{0,0.1,0.2,0.3,0.4\}$ we regress the corrupted observed-label component through
the origin on its clean counterpart, across candidates, subsets and runs; the predicted
slope is $a_\eta=1-\frac{K}{K-1}\eta$. We repeat it for mean-centred validation
accuracy, whose predicted slope is the same, since the derivation uses only linearity in
the label indicator.

Table~\ref{tab:attenuation} reports both, with a maximum absolute discrepancy of
$0.016$ at the lowest corruption rate and below $0.011$ elsewhere. Two consequences
follow. The corruption implementation does what the analysis says, so the corruption
results measure the intended quantity. And validation accuracy attenuates by the same
factor as the Brier alignment, so the identity holds for the whole class and not for
one statistic alone, which removes attenuation as a reason to prefer alignment over
noisy validation accuracy. It is not a performance guarantee.

With clean labels the coefficient has median $7.5$ on ResNet50 and $3.5$ on MobileNetV2,
reverting to zero in $11\%$ and $29\%$ of runs; under $40\%$ corruption on MobileNetV2 it
reaches the grid maximum in $53\%$, so the grid binds there; Supplement Table~\ref{tab:lambda} and Supplement Figure~\ref{fig:lambda}
give its distribution and the underlying regret curve. The frequency with which the procedure reverts to pure distortion is highest where the
labeled signal is least useful, so the coefficient does the work the construction
intends rather than sitting at a grid boundary.

\FloatBarrier

\section{Conclusion}
\label{sec:conclusion}
Selecting among fixed quantized variants of a shared teacher is a distinct problem from
building them, and today's estimators do not solve it. Minimum teacher distortion
selects the same eight-bit configuration in all $134$ candidate families, and that
candidate never minimizes target cross-entropy. Confidence-based estimators do worse than
uninformative: where quantization produces degenerate predictors more confident than the
healthy ones they select from the wrong end of the family, and no estimator we evaluate
detects which regime holds.

Anchoring converts that stability into a usable selector. With ten target labels it
reduces mean regret below direct validation in all four settings, and under corruption
in every one of the twenty budget--rate cells. The identities show the members differ
only in which function of the labeled sample is added, and the attenuation result gives
the class of criteria whose label signal decays at one rate. Neither bounds selection
regret, which we establish empirically.

The results also delimit the method. Anchoring is a general construction, not a property
of the teacher: a stronger anchor yields a better anchored selector. And where labels are plentiful
but unreliable a bounded accuracy criterion can do better, as on MobileNetV2.

\bibliography{refs}

\clearpage
\appendix
\section*{Supplementary Material}
\setcounter{section}{0}
\setcounter{table}{0}
\renewcommand{\thetable}{S\arabic{table}}
\makeatletter
\@ifundefined{theHsection}{}{\renewcommand{\theHsection}{supp.\arabic{section}}}
\@ifundefined{theHtable}{}{\renewcommand{\theHtable}{supp.\arabic{table}}}
\makeatother

\begin{table*}[!tb]
\centering
\scriptsize
\setlength{\tabcolsep}{4pt}
\caption{Notation used in the framework and the supplementary analyses.}
\label{tab:supp_notation}
\begin{tabular}{@{}lp{0.33\textwidth}lp{0.33\textwidth}@{}}
\toprule
Symbol & Meaning & Symbol & Meaning \\
\midrule
$f$, $p_f(\cdot\mid x)$ & fixed teacher and its predictive distribution
& $g\in\mathcal G$, $p_g(\cdot\mid x)$ & candidate and its predictive distribution \\
$\mathcal G$, $M$ & fixed candidate family and its cardinality
& $\mathcal U$ & unlabeled target input pool \\
$\mathcal S$, $n$ & labeled target sample and its size
& $\mathcal T$ & disjoint target test set, evaluation only \\
$e_y$ & one-hot vector for class $y$
& $z_f(x)$, $z_g(x)$ & teacher and candidate outputs in a chosen representation \\
$\delta_g^p(x)$ & probability change $p_g(x)-p_f(x)$
& $\Phi_{\mathcal S}(g)$ & supervised term added to the distortion \\
$D_{\mathcal V}^{d}(g,f)$ & teacher distance of type $d$ on input set $\mathcal V$
& $a_\ell(x,y;g,f)$ & per-example alignment contribution for loss $\ell$ \\
$A_{\mathcal S}(g,f)$ & alignment on $\mathcal S$
& $\widehat{\mathrm{Br}}_{\mathcal S}$ & empirical Brier score on $\mathcal S$ \\
$A_{\mathcal S}^{\mathrm{lab}}$, $A_{\mathcal S}^{\mathrm{teach}}$ & observed-label and teacher components
& $A_{\mathcal S}^{\pi}$ & permutation control, residual directions $e_y-p_f(x)$ permuted within $\mathcal S$ \\
$\alpha$ & weight of the supervised term in \eqref{eq:family}
& $K$ & number of classes \\
$\lambda$, $\lambda_{\max}$ & weight of the directional term and the largest value in its grid
& $\delta_g^z(x)$ & output change $z_g(x)-z_f(x)$ in the chosen representation \\
$L_{\mathrm{test}}(g)$ & held-out target cross-entropy
& $R(\hat g)$ & cross-entropy selection regret \\
$\eta$ & label-corruption probability
& $C(g)$ & deployment cost of candidate $g$ \\
$F$ & number of cross-validation folds
& $\mathcal G_B$ & candidates admissible under budget $B$ \\
\bottomrule
\end{tabular}
\end{table*}

Table~\ref{tab:supp_notation} collects the notation used throughout the paper
and this supplement.

This supplement is organised as follows, with the part of the main paper each
section supports. Supplement~A states and proves the two formal results
summarised in Section~3 and delimits what they do and do not bound. Supplement~B
gives the full experimental protocol behind Section~4, including the candidate
grid, the selector definitions, the coefficient search stated as
Algorithm~\ref{alg:anchored}, the seeding scheme and the run accounting.
Supplement~C reports the calibration audit and the candidate-family sensitivity
checks that accompany Table~2. Supplement~D collects the tables the main text
cites but does not reproduce, together with the additional analyses that
accompany them. Supplement~E quantifies
the two degradation regimes over all executions and is the source of the
confidence figures quoted in Section~4. Supplement~F documents the runs excluded
from the final cohort.

\section{Score Identities and Theoretical Boundaries}
\label{supp:identities}

\subsection{Smoothness interpretation}
Let $\delta_g^z(x)=z_g(x)-z_f(x)$. If $\ell(\cdot,y)$ has an $L$-Lipschitz
gradient along the segment joining $z_f(x)$ and $z_g(x)$, the descent lemma
gives
$\ell(z_g(x),y)-\ell(z_f(x),y)\le
\langle\nabla_z\ell(z_f(x),y),\delta_g^z(x)\rangle
+\frac{L}{2}\|\delta_g^z(x)\|_2^2$.
With $a_\ell(x,y;g,f)=-\langle\nabla_z\ell(z_f(x),y),\delta_g^z(x)\rangle$,
averaging over $\mathcal S$ gives
\[
\widehat L_{\mathcal S}(g)-\widehat L_{\mathcal S}(f)
\le
-A_{\ell,\mathcal S}(g,f)
+\frac{L}{2|\mathcal S|}\sum_{(x,y)\in\mathcal S}\|\delta_g^z(x)\|_2^2
\]
For probability outputs, pointwise Pinsker control gives
$\|\delta_g^p(x)\|_2^2\le\|\delta_g^p(x)\|_1^2
\le 2\,\mathrm{KL}(p_f(\cdot\mid x)\|p_g(\cdot\mid x))$, which provides direct empirical control when the KL and perturbation terms are averaged
over the same inputs. Our implemented selector averages the distortion over the whole
pool and the supervised term over $\mathcal S$, so the bound applies to the restriction
of both terms to a common input set and not to the score as evaluated. We present it as a local structural interpretation
rather than as practical control: the Lipschitz constant of the cross-entropy gradient
in probability space grows with the inverse of the probability floor, so the resulting
bound is numerically weak at the floors we use.

\subsection{Exact quadratic identity}
\begin{proposition}[Quadratic identity]\label{prop:quad}
Write $\delta_g^p(x)=p_g(x)-p_f(x)$, and on a labeled set $\mathcal S$ let
$D_{\mathcal S}^{(2)}(g,f)=|\mathcal S|^{-1}\sum\|\delta_g^p(x)\|_2^2$ and
$A_{\mathcal S}^{\mathrm B}(g,f)=|\mathcal S|^{-1}\sum\langle\delta_g^p(x),e_y-p_f(x)\rangle$.
If both terms are averaged over the same $\mathcal S$, then
\[
D_{\mathcal S}^{(2)}(g,f)-2A_{\mathcal S}^{\mathrm B}(g,f)
=\widehat{\mathrm{Br}}_{\mathcal S}(g)-\widehat{\mathrm{Br}}_{\mathcal S}(f)
\]
\end{proposition}
\noindent\emph{Proof.} Every labeled example satisfies
$\|p_g(x)-e_y\|_2^2-\|p_f(x)-e_y\|_2^2
=\|\delta_g^p(x)\|_2^2-2\langle\delta_g^p(x),e_y-p_f(x)\rangle$;
averaging over $\mathcal S$ gives the claim. \hfill$\square$

Evaluating both sides candidate by candidate over the 90 CNN runs,
the two agree to $1.3\times10^{-15}$, that is, to numerical precision.
The teacher term is candidate-independent, so with $\lambda=2$ and both terms
on the same inputs the criterion selects exactly the same candidate as
empirical Brier validation. This is an identity between selection criteria,
presented as a consistency check rather than a new scoring rule.

\subsection{Observed-label and teacher components}
By linearity, $A_{\mathcal S}^{\mathrm B}
=A_{\mathcal S}^{\mathrm{lab}}+A_{\mathcal S}^{\mathrm{teach}}$ with
$A_{\mathcal S}^{\mathrm{lab}}=|\mathcal S|^{-1}\sum\langle\delta_g^p(x),e_y\rangle$ and
$A_{\mathcal S}^{\mathrm{teach}}=-|\mathcal S|^{-1}\sum\langle\delta_g^p(x),p_f(x)\rangle$.
For any common input set $\mathcal V$,
\begin{equation*}
\begin{split}
&D_{\mathcal V}^{(2)}(g,f)-2A_{\mathcal V}^{\mathrm{teach}}(g,f)\\
&\qquad=\frac{1}{|\mathcal V|}\sum_{x\in\mathcal V}
\left[\|p_g(x)\|_2^2-\|p_f(x)\|_2^2\right]
\end{split}
\end{equation*}
so minimizing the canonical fully quadratic teacher-only criterion favors the
candidate with the smallest average squared predictive norm. Taking
$\mathcal V=\mathcal U$ gives a label-free predictive-concentration criterion.
This does not imply that the empirical KL-based teacher-component ablation is
identical to concentration ranking, since that ablation retains KL distortion
and a selected coefficient.

\subsection{Attenuation under symmetric corruption}
\begin{proposition}[Common attenuation]\label{prop:atten}
Let a label-dependent criterion be $a_g(x)^\top e_y$ and let each label be replaced,
with probability $\eta$, by a class drawn uniformly from the $K-1$ incorrect ones. Then
\[
\mathbb E\!\left[a_g^\top e_{\widetilde y}\mid y\right]
=\left(1-\tfrac{K}{K-1}\eta\right)a_g^\top e_y+\tfrac{\eta}{K-1}\,\mathbf 1^\top a_g
\]
If $\mathbf 1^\top a_g$ does not vary across candidates, the second term is a common
offset and candidate comparisons are multiplied by the single factor
$1-\frac{K}{K-1}\eta$. The condition holds for teacher contrasts, whose coordinates sum
to zero, and for accuracy indicators, whose coordinates sum to one; it fails for
cross-entropy, whose coefficient vector $-\log p_g$ has a candidate-dependent
coordinate sum.
\end{proposition}
\noindent\emph{Proof.} Conditional on the clean class $y$, $\mathbb E[e_{\widetilde y}\mid y]
=(1-\frac{K}{K-1}\eta)e_y+\frac{\eta}{K-1}\mathbf 1$; taking the inner product with
$a_g$ gives the display. \hfill$\square$

Applying Proposition~\ref{prop:atten} to the alignment statistic, where
$a_g=p_g(x)-p_f(x)$ so that $\langle a_g,\mathbf 1\rangle=0$,
\[
\mathbb E\!\left[A_{\mathcal S}^{\mathrm{noisy}}\mid\mathcal S\right]
=\left(1-\tfrac{K}{K-1}\eta\right)A_{\mathcal S}^{\mathrm{lab}}
+A_{\mathcal S}^{\mathrm{teach}}
\]
The factor is positive whenever $\eta<(K-1)/K$; for $K=30$ this bound is
$29/30\approx0.967$, outside the range we evaluate, so within our experiments
the label signal shrinks by a known factor and never inverts. Both slopes are
verified empirically in Table~\ref{tab:attenuation}, to within $0.016$.

\paragraph{The identity is not specific to Brier alignment.}
The derivation uses only that $A_{\mathcal S}^{\mathrm{lab}}$ is linear in $e_y$ and that
the candidate--teacher difference is orthogonal to $\mathbf 1$. Any functional
linear in $e_y$ obeys the same attenuation, including empirical accuracy:
\[
\mathbb E[\widehat{\mathrm{Acc}}^{\mathrm{noisy}}(g)]
=\left(1-\tfrac{K}{K-1}\eta\right)\widehat{\mathrm{Acc}}(g)
+\tfrac{\eta}{K-1}
\]
affine in clean accuracy with the same positive slope, so noisy validation
accuracy preserves the clean accuracy ordering in expectation.
Table~\ref{tab:attenuation} verifies both slopes on the same runs.

\section{Experimental Protocol}
\label{supp:protocol}

\subsection{Runs, executions, and resampling units}
Four nested levels appear throughout, and Supplement Figure~\ref{fig:design} shows how
they relate. Table~\ref{tab:supp_runs} distinguishes independent candidate-family runs from
within-run repetitions and from whole-pipeline re-executions.

\begin{figure}[!tb]
\centering
\begin{tikzpicture}[font=\small,>=stealth,x=1cm,y=1cm]
\node[anchor=east,font=\small\bfseries] at (2.05,3.05) {base configuration};
\node[anchor=east,font=\footnotesize,black!60] at (2.05,2.76) {shift $\times$ seed};
\node[anchor=east,font=\small\bfseries] at (2.05,1.85) {execution};
\node[anchor=east,font=\footnotesize,black!60] at (2.05,1.56) {retrained teacher};
\node[anchor=east,font=\small\bfseries] at (2.05,0.72) {family-run};
\node[anchor=east,font=\footnotesize,black!60] at (2.05,0.43) {statistical unit};
\node[anchor=east,font=\small\bfseries] at (2.05,-0.42) {outcome};
\node[anchor=east,font=\footnotesize,black!60] at (2.05,-0.71) {repeated subsets};
\draw[rounded corners=1.5pt,black!45,fill=black!4] (2.35,2.62) rectangle (4.75,3.18);
\node[font=\footnotesize] at (3.55,2.90) {CNN: $15{+}15=30$};
\draw[rounded corners=1.5pt,black!45,fill=black!4] (5.05,2.62) rectangle (7.45,3.18);
\node[font=\footnotesize] at (6.25,2.90) {ViT: $17{+}27=44$};
\draw[->,black!55] (3.55,2.62) -- (3.55,2.02);
\draw[->,black!55] (6.25,2.62) -- (6.25,2.02);
\node[anchor=west,font=\footnotesize,black!60] at (3.62,2.32) {$\times\,3$};
\node[anchor=west,font=\footnotesize,black!60] at (6.32,2.32) {$\times\,1$};
\draw[rounded corners=1.5pt,black!45,fill=black!8] (2.35,1.46) rectangle (4.75,2.02);
\node[font=\footnotesize] at (3.55,1.74) {$30\times3=90$};
\draw[rounded corners=1.5pt,black!45,fill=black!8] (5.05,1.46) rectangle (7.45,2.02);
\node[font=\footnotesize] at (6.25,1.74) {$44\times1=44$};
\draw[->,black!55] (3.55,1.46) -- (4.35,0.90);
\draw[->,black!55] (6.25,1.46) -- (5.45,0.90);
\draw[rounded corners=1.5pt,black!45,fill=black!14] (2.35,0.34) rectangle (7.45,0.90);
\node[font=\footnotesize] at (4.90,0.62) {$90+44=134$ teacher--family pairs};
\draw[->,black!55] (4.90,0.34) -- (4.90,-0.22);
\draw[rounded corners=1.5pt,black!45,fill=black!4] (2.35,-0.78) rectangle (7.45,-0.22);
\node[font=\footnotesize] at (4.90,-0.50) {labeled subsets of size $n$};
\end{tikzpicture}
\caption{Levels of the experimental design. A base configuration is a shift--seed pair:
$30$ on the two CNN settings and $44$ on the two ViT settings, $74$ in all. Each CNN
configuration is executed three times with a retrained teacher on identical splits and
each ViT configuration once, so the number of independently trained teachers, and hence
of \emph{family-runs}, is $30\times3+44\times1=134$: $45$ per CNN architecture, $17$
on ViT--DomainNet and $27$ on ViT--CIFAR-20. The family-run is the statistical unit: subset outcomes
are averaged within it. For paired inference the three executions of a configuration
are averaged as well, leaving $15$ paired units per CNN architecture.}
\label{fig:design}
\end{figure}
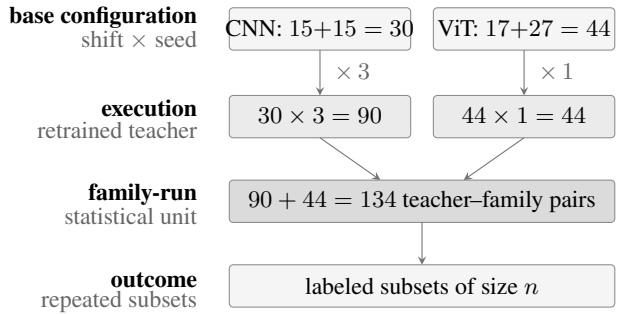

\begin{table}[h]
\centering
\scriptsize
\setlength{\tabcolsep}{2pt}
\caption{Accounting of the experimental design. A run contains one teacher and
its fixed $72$-candidate family. Executions are independent end-to-end
repetitions with retrained teachers on identical splits. Results for the two CNN
settings pool all three executions, giving $45$ runs per architecture; the two ViT
settings have one execution each. $^{\dagger}$for $n<300$; the complete pool is
evaluated once at $n=300$. On the two ViT settings every clean budget is below
the pool size, so all of them use the full $25$ repetitions. The counts in this
table are those of the budget sweep; the budget-by-rate grid of the main text
draws ten repetitions per cell at every rate, including $\eta=0$.
The ViT--DomainNet selection pool varies by shift, from $2526$ to $6308$
examples; the largest supervised budget remains $n=150$.}
\label{tab:supp_runs}
\begin{tabular}{lcccccc}
\toprule
Setting & Cfg.\ & Ex.\ & Runs & Pool & Clean & Noise\ \\
\midrule
RN50--DN & 15 & 3 & 45 & $150{+}150$ & 25$^{\dagger}$ & 10 \\
MNV2--DN & 15 & 3 & 45 & $150{+}150$ & 25$^{\dagger}$ & 10 \\
ViT--DN & 17 & 1 & 17 & $2526$--$6308$ & 25 & 10 \\
ViT--C20 & 27 & 1 & 27 & $666{+}666$ & 25 & 10 \\
\midrule
Total & 74 & -- & 134 & -- & -- & -- \\
\bottomrule
\end{tabular}
\end{table}

\paragraph{Clarification of the main-text summary.}
The figure $3646$ quoted in Section~4 is the pool size for one ViT--DomainNet
shift; across the complete $17$-run cohort the selection-pool size varies from
$2526$ to $6308$. Each run uses its own complete pool, and the largest
supervised budget is fixed at $n=150$ throughout.

\paragraph{Candidate construction.}
Quantization is weight-only and simulated in floating point, so no integer
kernel is required. For a weight tensor $w$, bit width $b$, and clipping
percentile $q$, the scale is
\[
s=\frac{\operatorname{pct}_{q}(|w|)}{2^{b-1}-1}
\]
We use the symmetric signed range $[-(2^{b-1}-1),\,2^{b-1}-1]$, zero point zero,
and round-half-to-even. Weights are clipped to $\pm\operatorname{pct}_{q}(|w|)$
before scaling. Per-tensor quantization uses one scale for the complete tensor.
Per-channel quantization reduces over all axes except the last and uses one
scale per output channel. Only kernel weights are quantized; biases,
normalization parameters, and embeddings remain in floating point. When endpoint
skipping is enabled, the first and last quantizable layers also remain in
floating point.

Write the ordered arrays $B=(2,3,4,5,6,8)$, $Q=(99.0,99.5,100.0)$ and
$G=(\mathrm{tensor},\mathrm{channel})$, and let
\[
\begin{aligned}
\mathcal B&=\{2,3,4,5,6,8\},&
\mathcal Q&=\{99.0,99.5,100.0\},\\
\mathcal G&=\{\mathrm{tensor},\mathrm{channel}\},&
\mathcal E&=\{0,1\}
\end{aligned}
\]
The candidate family is $\mathcal B\times\mathcal Q\times\mathcal G\times
\mathcal E$, giving $6\times3\times2\times2=72$ candidates per teacher. The
endpoint indicator varies fastest, followed by granularity, clipping percentile,
and bit width. For zero-based index $i$,
\[
\begin{aligned}
b_i &= B_{\lfloor i/12\rfloor},&
q_i &= Q_{\lfloor i/4\rfloor\bmod 3},\\
g_i &= G_{\lfloor i/2\rfloor\bmod 2},&
e_i &= i\bmod 2
\end{aligned}
\]
Here, $e_i=1$ means that the first and last quantizable layers remain in
floating point.

The selection pool is the union of the quantization-calibration split and the
selection split. Because clipping percentiles are computed from the weight
tensors rather than from data, the calibration split plays no role in
constructing the candidate family, and its inclusion in the selection pool does
not leak candidate-construction information. We report the composition
explicitly because the two halves are not interchangeable: on a $150$-example
half-pool the distortion ordering between the two $8$-bit per-channel unclipped
candidates can reverse.

Test splits contain $900$ examples per run for the DomainNet CNNs, $1824$ for
ViT--DomainNet and $668$ for ViT--CIFAR-20. All test sets are fixed, disjoint
from the selection pools, and used only after candidate selection.

\paragraph{Teacher training.}
CNN teachers: transfer learning from the framework's ImageNet weights and its own
preprocessing, five epochs of head training at learning rate $10^{-3}$ with the
backbone frozen, then three unfrozen at $10^{-5}$, both with Adam and batch size
$64$. ViT teachers: continual distillation over a
\texttt{vit\_base\_patch16\_224} backbone \citep{dosovitskiy2021image} at
$224\times224$, Adam at $10^{-4}$, batch size $64$, distillation temperatures
$2.0$ and $10.0$. Seeds $0$, $1$ and
$2$ index the repetitions, each set once per run for the Python, NumPy and
framework generators.

\paragraph{Estimator variants.}
SoftmaxCorr and COT are evaluated with a uniform class prior, which the balanced
target splits permit; this is the variant reported throughout.

\paragraph{Calibration of the label-free estimators.}
Two of the nine need a calibration set: ATC learns a threshold and DOC a reference
accuracy, both defined on labeled data in their original form. Our pipelines retain no
source-domain predictions, so we calibrate them on the target calibration split using
the teacher's arg-max as the reference label, which keeps them label-free and uses the
reference this setting already provides. One further estimator is outside
the comparison for a structural reason: \citet{li2026frap} requires an external
foundation model as its reference, which changes the resource assumption of this
setting, where the teacher is the reference.

\subsection{Selector definitions and coefficient selection}
Table~\ref{tab:supp_selectors} lists all selectors.

\paragraph{Implementation convention.}
In all principal experiments $D_{\mathcal U}^{\mathrm{KL}}$ is computed on the
complete selection-input pool and remains fixed across label budgets and
cross-validation folds; only the supervised component is evaluated on
$\mathcal S$. The same-$\mathcal S$ construction appears only in the quadratic
identity of Proposition~\ref{prop:quad} and in the anchor-placement sensitivity
analysis of Supplement~C. No candidate is retrained or adapted at any point, and
the held-out test set is never touched. Five
folds are used when $n\ge 25$ and leave-one-out below; folds are contiguous
blocks of the labeled sample in its stored order and become singletons when
$n<F$. Every draw that varies within a run is seeded deterministically from the
quantity that indexes it, so a replication recovers the same labeled subsets and
the same corruption realizations without storing them: the subset of size $n$ at
repetition $r$ uses seed $10^{5}+7919n+r$, and the corruption realization at
rate $\eta$ and repetition $r$ uses seed $9500+\lfloor 100\eta\rceil+r$, each
initializing a PCG64 generator. Write $\mathcal S=\bigcup_{k=1}^{F}\mathcal S_k$ for the fold partition and
$\mathcal S_{-k}$ for its complement. For a coefficient $c$ in the grid
$\mathcal C$, let $\hat g_{c,-k}$ be the candidate the selector returns when its
supervised term is averaged over $\mathcal S_{-k}$ only, the distortion term being
unchanged because it uses no labels. The retained coefficient is
\[
\hat c=\arg\min_{c\in\mathcal C}\;
\frac{1}{F}\sum_{k=1}^{F} L^{\mathrm{cv}}\!\left(\hat g_{c,-k},\mathcal S_k\right)
\]
with exact ties resolved by the first value in the ordered grid, and the reported
candidate is the one the selector returns on all of $\mathcal S$ at $\hat c$.
Algorithm~\ref{alg:anchored} states the whole procedure.

\begin{algorithm}[!tb]
\caption{Anchored selection. The anchor is evaluated once on the complete pool
and reused at every budget and fold; only the supervised term sees
$\mathcal S$. Written for an added supervised term, as in CE-combo; the
directional variants replace $c\,\phi$ by $-c\,a$ and are otherwise identical.}
\label{alg:anchored}
\footnotesize
\begin{algorithmic}[1]
\REQUIRE teacher $f$, candidate family $\mathcal G$, unlabeled pool $\mathcal U$,
labeled sample $\mathcal S$ of size $n$, coefficient grid $\mathcal C$ with
$0\in\mathcal C$, fold count $F$
\ENSURE selected candidate $\hat g$
\STATE $D(g)\leftarrow \frac{1}{|\mathcal U|}\sum_{x\in\mathcal U}
       \mathrm{KL}\!\left(p_f(\cdot\mid x)\,\|\,p_g(\cdot\mid x)\right)$
       for every $g\in\mathcal G$ \COMMENT{once per run; no labels}
\STATE $\phi(g,x,y)\leftarrow$ per-example supervised statistic
\STATE partition $\mathcal S$ into contiguous folds
       $\mathcal S_1,\dots,\mathcal S_F$; $F=5$ if $n\ge 25$, else
       $F=\min(n,10)$, which is leave-one-out at every budget we evaluate
\FOR{$c \in \mathcal C$ in increasing order}
  \STATE $\ell(c)\leftarrow 0$
  \FOR{$k = 1$ \TO $F$}
    \STATE $\hat g_{c,-k}\leftarrow \arg\min_{g\in\mathcal G}\;
           D(g)+c\,\frac{1}{|\mathcal S_{-k}|}\sum_{(x,y)\in\mathcal S_{-k}}\phi(g,x,y)$
    \STATE $\ell(c)\leftarrow \ell(c)+\frac{1}{|\mathcal S_k|}
           \sum_{(x,y)\in\mathcal S_k}-\log p_{\hat g_{c,-k}}(y\mid x)$
  \ENDFOR
\ENDFOR
\STATE $\hat c\leftarrow \arg\min_{c\in\mathcal C}\ell(c)$, ties to the first grid value
\STATE \textbf{return} $\hat g \leftarrow \arg\min_{g\in\mathcal G}\;
       D(g)+\hat c\,\frac{1}{n}\sum_{(x,y)\in\mathcal S}\phi(g,x,y)$
\end{algorithmic}
\end{algorithm}

 The
held-out loss $L^{\mathrm{cv}}$ is the selector's own supervised criterion:
empirical cross-entropy for CE-combo, and cross-entropy for alignment and for the
permutation control, whose directional statistic is a score component rather than a
loss. Under corruption the labels in
both $\mathcal S_{-k}$ and $\mathcal S_k$ are the corrupted ones, so no selector
sees a clean label at any stage. The
anchored-term grid is $\alpha\in\{0,0.05,0.1,0.25,0.5,1,2,4,8,16,32,64\}$ ($12$
values) and the alignment grid is $\lambda\in\{0,0.5,\ldots,10\}$ ($21$ values).
Both contain zero, so an anchored selector can revert to pure distortion; because
$\alpha$ multiplies the supervised term, large $\alpha$ approaches direct
validation in the limit. On the finite grid the largest value is selected in
$3.3\%$, $2.7\%$ and $10.0\%$ of outcomes at ten, fifty and all labels, and the selector
at $\alpha_{\max}=64$ returns the direct-validation candidate in $98.0\%$, $95.3\%$ and
$100\%$ of them, so the grid spans the interpolation to that resolution.
Every mechanism variant selects its own coefficient through the identical
procedure, so differences between them are not attributable to tuning freedom.

\begin{table*}[!tb]
\centering
\scriptsize
\setlength{\tabcolsep}{4pt}
\caption{Selectors and ablations. All supervised entries receive the same
labeled sample $\mathcal S$ of the stated size, including labels consumed by
coefficient selection.}
\label{tab:supp_selectors}
\begin{tabular}{lll}
\toprule
Selector & Score minimized over $g\in\mathcal G$ & Role \\
\midrule
Distortion & $D_{\mathcal U}^{\mathrm{KL}}(g,f)$
& label-free teacher-relative baseline \\
Highest precision & none (fixed configuration)
& data-free baseline matching the distortion pick \\
Validation CE & $\widehat{\mathrm{CE}}_{\mathcal S}(g)$
& direct supervised reference \\
CE-combo & $D_{\mathcal U}^{\mathrm{KL}}(g,f)
+\alpha\widehat{\mathrm{CE}}_{\mathcal S}(g)$ & anchored cross-entropy \\
Alignment & $D_{\mathcal U}^{\mathrm{KL}}(g,f)
-\lambda A_{\mathcal S}^{\mathrm B}(g,f)$
& teacher-relative directional selector \\
Teacher component & $D_{\mathcal U}^{\mathrm{KL}}(g,f)-\lambda A_{\mathcal S}^{\mathrm{teach}}(g,f)$
& mechanism ablation, no observed class \\
Permutation control & $D_{\mathcal U}^{\mathrm{KL}}(g,f)
-\lambda A_{\mathcal S}^{\pi}(g,f)$
& destroys residual-direction--input correspondence \\
\bottomrule
\end{tabular}
\end{table*}

Supplement Table~\ref{tab:settings} lists the four settings and their sizes.

\begin{table}[h]
\centering
\scriptsize
\setlength{\tabcolsep}{2.2pt}
\caption{Experimental settings. A run contains one fixed teacher and its
$72$-candidate family. The selection pool is the union of the
quantization-calibration and selection splits; the largest label budget is listed
separately because it is not the pool size. Exec.\ counts independent end-to-end
repetitions of the whole pipeline with retrained teachers on identical splits.}
\label{tab:settings}
\begin{tabular}{llcccccc}
\toprule
Data & Teacher & Cls.\ & Cfg.\ & Ex.\ & Runs & Pool & Max.\ $n$ \\
\midrule
DomainNet & ResNet50 & 30 & 15 & 3 & 45 & 300 & 300 \\
DomainNet & MNV2 & 30 & 15 & 3 & 45 & 300 & 300 \\
DomainNet & ViT & 50 & 17 & 1 & 17 & 2.5k--6.3k & 150 \\
CIFAR-20 & ViT & 20 & 27 & 1 & 27 & 1332 & 200 \\
\midrule
Total & -- & -- & 74 & -- & 134 & -- & -- \\
\bottomrule
\end{tabular}
\end{table}

\paragraph{How the medians of the main-text tables are taken.}
The medians reported for the three selectors across budgets are computed after
averaging the repeated labeled subsets within each family-run: they are medians
of run means rather than medians of the pooled run--subset outcomes. The two
differ because the median does not commute with the mean, and the pooled median
is lower at the smallest budgets. The run-level convention is the one stated in
Supplement~B and the one the paired inference uses.

\paragraph{When the anchor is already optimal.}
The label-free pick is not uniformly suboptimal. Read run by run, the candidate
that minimizes teacher distortion attains zero cross-entropy regret in $15$ of
the $134$ families, ten of them on MobileNetV2, three on ResNet50 and two on
ViT--DomainNet, and is non-oracle in the remaining $119$. The aggregate regrets
reported elsewhere are means over those two groups, so the label-free rule is
best read as usually, but not always, leaving regret on the table.

\paragraph{Inferential unit.}
Figure~\ref{fig:design} summarizes the design. Unless
a caption explicitly identifies a pooled run--subset or run--realization
distribution, descriptive means and medians are computed after averaging
repeated outcomes within each family-run. Tables reporting upper tails or
operational outcome distributions pool the repetitions, as their captions state.
Formal inference never treats subsets, candidates or corruption realizations as
independent units. Descriptive
quantities, all means and medians reported for the CNN settings, are computed over
the $45$ teacher--family runs, that is five shifts by three seeds by three
executions, and over $17$ runs on ViT--DomainNet and $27$ on ViT--CIFAR-20.
Paired inference is
computed after averaging the three executions within each shift--seed configuration,
which leaves $15$ paired units per CNN architecture, because the three executions
share splits and are not independent replicates of the same comparison. Neither
level counts candidates or resample outcomes as units.

\paragraph{Where the between-run variation comes from.}
The standard deviations reported in Table~\ref{tab:sota} are large relative to their
means, and a variance
decomposition shows why. Taking CE-combo at ten labels and grouping the $45$ runs by
target shift, the between-shift component accounts for $90\%$ of the total variance
on ResNet50 and $74\%$ on MobileNetV2. The dispersion is therefore mostly
heterogeneity across domain shifts, a property of the benchmark, rather than
instability of the selector on a given shift.

\paragraph{Clustering.}
The $15$ base configurations of each CNN setting are five shifts crossed with three seeds, and
three of the five shifts share Sketch as target domain, so the runs are not
fully exchangeable. Run-level bootstrap intervals should be read as having
approximately five effective clusters per architecture, and the exact
sign-flip test on $15$ units has limited power. We therefore treat per-setting
non-significance as inconclusive rather than as evidence of no effect, and treat the
per-setting values as the primary evidence. Supplement Table~\ref{tab:paired} reports
the resulting paired comparisons between supervised selectors.

\paragraph{What each pipeline stores.}
Every pipeline stores predictive distributions for the teacher and each of its
candidates on all three splits. Every estimator we screen is defined on those
distributions and is evaluated on all four settings.

\paragraph{Computing infrastructure.}
Experiments were run in Google Colab notebooks on Python $3.12.13$ under Ubuntu
$22.04.5$ LTS, kernel $6.6.122$, with results written to Google Drive. Colab
allocates the host per session, so the runtime is a class rather than a fixed
machine: the sessions used $2$ to $8$ vCPUs of an Intel Xeon at
$2.0$--$2.2$\,GHz with $12$ to $50$\,GB of system memory, and an NVIDIA
accelerator, observed as A100-SXM4 with $80$\,GB and Tesla T4 with $15$\,GB,
driver $580.82.07$ and CUDA toolkit $12.8$. The two DomainNet CNN pipelines use
TensorFlow $2.20.0$, with teachers initialized from ImageNet weights supplied by
the framework's application module; the two ViT pipelines use PyTorch $2.11.0$
with torchvision $0.26.0$ and timm $1.0.28$. Quantization is simulated in floating
point and applies to weights only, so no integer kernel or quantization-specific
runtime is required. The selection and analysis stage consumes only the stored
probability matrices and runs on CPU with NumPy $2.0.2$ and SciPy $1.16.3$: the
principal selector tables recompute from the archived outcome files and individual
runs from the stored predictions, both without an accelerator, while the
full-cohort diagnostics and the deployment-cost analysis need the complete cache
collection or architecture-level metadata.

\section{Calibration and Candidate-Family Sensitivity}
\label{supp:diagnostics}

On ResNet50, validation CE selects candidates with substantially lower ECE and
cross-entropy but lower top-$1$ accuracy; on MobileNetV2 it reduces both while
producing a small positive accuracy difference. This explains why validation
accuracy fails as a corruption-robust criterion on ResNet50 specifically: the
quantity that must be recovered there is invisible to a $0$--$1$ loss.

Distortion remains poor on ResNet50 after restricting the family to
high-precision candidates or to the top $20$ by distortion, and its median
regret rises to $0.471$ inside the low-precision family, so the selection
difficulty is not an artifact of a single catastrophic low-bit candidate. The
comparison does not show uniform dominance of any anchored method across settings
and budgets.



\begin{table}[!tb]
\centering\scriptsize\setlength{\tabcolsep}{3pt}
\caption{Paired run-level inference. Differences are computed after averaging
within-run outcomes; negative favors the first method. $p$ is the two-sided sign-flip value, obtained by complete enumeration of the
sign assignments when the number of paired units is at most $22$ and by a
$200{,}000$-draw Monte Carlo sign-flip test otherwise, Holm-adjusted across the
comparisons shown in each row, which is the family; boldface marks $p<0.05$ after adjustment. Upper block: clean labels.
Lower block: symmetric corruption.}
\label{tab:paired}
\begin{tabular}{llrlrlrl}
\toprule
& & \multicolumn{2}{c}{CE-cmb $-$ Val-CE} & \multicolumn{2}{c}{Align $-$ Val-CE}
& \multicolumn{2}{c}{Align $-$ CE-cmb} \\
\cmidrule(lr){3-4}\cmidrule(lr){5-6}\cmidrule(lr){7-8}
Setting & $n$ & Diff. & $p$ & Diff. & $p$ & Diff. & $p$ \\
\midrule
RN50--DN & 10 & -0.016 & 1.000 & -0.014 & 1.000 & +0.001 & 1.000 \\
 & 25 & +0.003 & 0.776 & +0.020 & 0.469 & +0.016 & 0.077 \\
 & 50 & +0.008 & 0.393 & +0.026 & 0.112 & +0.018 & 0.063 \\
 & 100 & +0.004 & 0.485 & +0.020 & 0.167 & +0.016 & \textbf{0.029} \\
 & 150 & -0.001 & 0.899 & +0.012 & 0.324 & +0.012 & 0.324 \\
 & 300 & -0.006 & 0.879 & -0.001 & 0.879 & +0.005 & 0.879 \\
\midrule
MNV2--DN & 10 & -0.124 & $\mathbf{<}$\textbf{0.001} & -0.139 & $\mathbf{<}$\textbf{0.001} & -0.015 & 0.111 \\
 & 25 & -0.039 & $\mathbf{<}$\textbf{0.001} & -0.037 & $\mathbf{<}$\textbf{0.001} & +0.002 & 0.448 \\
 & 50 & -0.012 & 0.158 & -0.009 & 0.188 & +0.004 & 0.215 \\
 & 100 & -0.001 & 1.000 & -0.000 & 1.000 & +0.001 & 1.000 \\
 & 150 & +0.001 & 1.000 & +0.000 & 1.000 & -0.001 & 1.000 \\
 & 300 & +0.004 & 1.000 & +0.003 & 1.000 & -0.000 & 0.938 \\
\midrule
Setting & $\eta$ & Diff. & $p$ & Diff. & $p$ & & \\
\midrule
RN50--DN & 20\% & $-0.010$ & $\mathbf{<}$\textbf{0.001} & $-0.011$ & 0.102 & & \\
 & 40\% & $-0.032$ & $\mathbf{<}$\textbf{0.001} & $-0.051$ & $\mathbf{<}$\textbf{0.001} & & \\
MNV2--DN & 20\% & $-0.076$ & \textbf{0.013} & $-0.166$ & \textbf{0.013} & & \\
 & 40\% & $-0.061$ & \textbf{0.012} & $-0.431$ & \textbf{0.002} & & \\
ViT--DN & 20\% & $-0.043$ & 0.248 & $-0.021$ & 0.768 & & \\
 & 40\% & $-0.016$ & 0.123 & $-0.295$ & \textbf{0.036} & & \\
ViT--C20 & 20\% & $-0.069$ & \textbf{0.006} & $-0.000$ & 0.913 & & \\
 & 40\% & $+0.000$ & 1.000 & $+0.000$ & 1.000 & & \\
\bottomrule
\end{tabular}
\end{table}

\subsection{Anchor placement and anchor choice}
\label{supp:anchorvar}
Two design questions can be answered by re-running the same protocol with one
element changed, and Supplement Table~\ref{tab:anchorvar} does so. Evaluating the teacher
distortion on the complete unlabeled pool, on the labeled subsample, or on an
independent subsample of equal size gives regrets that differ by at most $0.007$,
well inside the run-to-run standard deviation, so nothing in our conclusions turns
on that choice even though the operational argument favours using all available
inputs. Replacing the teacher anchor with an optimal-transport anchor is more
informative: it improves on the teacher on MobileNetV2, from $0.107$ to $0.085$ at
ten labels, and worsens it on ResNet50, from $0.163$ to $0.189$. Anchoring is
therefore a general construction rather than a property of the teacher, and each
anchor helps where it is itself the stronger label-free rule; the teacher's
attraction is that it is immune to the confidence-collapse mode that defeats the
confidence-based family.

\begin{table}[!tb]
\centering\scriptsize\setlength{\tabcolsep}{5pt}
\caption{Accuracy regret, the quantity several of the label-free estimators were
proposed to rank, over all three executions and $45$ runs per architecture.
Entries are the gap in target top-$1$ accuracy between the best candidate and the one
selected. The deployment objective in this paper is cross-entropy, so these estimators
are elsewhere evaluated outside their original objective; on their own metric the
confidence-based group still fails by two orders of magnitude. The anchored
selectors spend labels, so they are outside this label-free comparison.}
\label{tab:accreg}
\begin{tabular}{@{}lcc@{}}
\toprule
Estimator & RN50--DN & MNV2--DN \\
\midrule
Distortion & 0.0032 & 0.0081 \\
SoftmaxCorr & 0.0092 & 0.0085 \\
Nuclear norm & 0.0090 & 0.0031 \\
DOC & 0.0033 & 0.0076 \\
ATC-MC & 0.3211 & 0.2717 \\
ATC-NE & 0.3211 & 0.2404 \\
Entropy & 0.3178 & 0.3101 \\
Average confidence & 0.3178 & 0.3101 \\
\bottomrule
\end{tabular}
\end{table}

\subsection{Clipping and metric sensitivity}
\label{supp:sens}
Supplement Table~\ref{tab:sens} reports two checks. Varying the probability floor across
three values spanning six orders of magnitude moves distortion's regret by at most $0.0003$, so the large
regret magnitudes are a property of the candidate families and not of the floor.
Supplement Table~\ref{tab:accreg} evaluates the label-free screen on accuracy
regret rather than cross-entropy regret, addressing the objection that several
estimators were designed to rank accuracy: the
anchor incurs $0.003$ and $0.008$ in accuracy regret while average confidence incurs
$0.318$ and $0.310$, so the failure is not an artifact of scoring them on an
objective they were not built for.

\begin{table}[!tb]
\centering\scriptsize\setlength{\tabcolsep}{3pt}
\caption{Where the anchor is evaluated, and which anchor. Mean regret over the first complete execution of each CNN setting, $15$ runs per
architecture, so the values are not directly comparable with the pooled means of
Supplement Table~\ref{tab:sweep}. Columns two to four vary only the input set on which
the teacher distortion is averaged: the complete unlabeled pool, the labeled subsample itself, or an
independent subsample of the same size. The last column replaces the teacher anchor
with the optimal-transport estimator, rescaled so that the coefficient grid remains
comparable.}
\label{tab:anchorvar}
\begin{tabular}{@{}llccccc@{}}
\toprule
Setting & $n$ & Val-CE & All $\mathcal U$ & Same $\mathcal S$ & Indep. & COT anchor \\
\midrule
ResNet50--DomainNet & 10 & 0.1801 & 0.1627 & 0.1564 & 0.1629 & 0.1888 \\
 & 50 & 0.0608 & 0.0677 & 0.0682 & 0.0672 & 0.0757 \\
 & 300 & 0.0129 & 0.0086 & 0.0086 & 0.0086 & 0.0094 \\
\midrule
MobileNetV2--DomainNet & 10 & 0.2473 & 0.1065 & 0.1080 & 0.1101 & 0.0847 \\
 & 50 & 0.0504 & 0.0231 & 0.0227 & 0.0219 & 0.0119 \\
 & 300 & 0.0010 & 0.0010 & 0.0010 & 0.0010 & 0.0009 \\
\bottomrule
\end{tabular}
\end{table}

Supplement Table~\ref{tab:sens} collects both checks.

\begin{table}[!tb]
\centering\scriptsize\setlength{\tabcolsep}{4pt}
\caption{Two sensitivity checks. Left: distortion's cross-entropy regret under three
probability floors, showing that the magnitudes do not depend on clipping. Right:
regret in \emph{accuracy}, the quantity the confidence estimators were designed to
rank, for the anchor and for average confidence. The confidence estimator fails on
its own metric as well.}
\label{tab:sens}
\begin{tabular}{@{}lccc cc@{}}
\toprule
& \multicolumn{3}{c}{CE regret at floor} & \multicolumn{2}{c}{Accuracy regret} \\
\cmidrule(lr){2-4}\cmidrule(lr){5-6}
Setting & $10^{-6}$ & $10^{-8}$ & $10^{-12}$ & Dist.\ & Avg.\ conf.\ \\
\midrule
ResNet50--DomainNet & 0.2158 & 0.2161 & 0.2161 & 0.0040 & 0.3212 \\
MobileNetV2--DomainNet & 0.0848 & 0.0848 & 0.0848 & 0.0076 & 0.3076 \\
\bottomrule
\end{tabular}
\end{table}

\section{Additional Experimental Results}

\paragraph{Budget sweep and upper tail.}
Supplement Table~\ref{tab:sweep} gives every clean-label budget we evaluate. Reading down
a setting shows the pattern the main text summarizes: anchoring wins at ten and twenty-five labels and loses from fifty onward, which is
where the eleven-of-twenty-three count in the main text comes from. Supplement Table~\ref{tab:tail} gives the corresponding upper tail. The
$95$th percentile and the catastrophic-selection rate move together with the mean at
the smallest budgets and separate at the largest, where anchoring can leave the mean
unchanged while still removing the worst outcomes.
\begin{table}[!tb]
\centering
\scriptsize
\setlength{\tabcolsep}{4pt}
\caption{Budget ablation: mean target cross-entropy regret at every clean-label
budget evaluated. Distortion is label-free and constant in $n$. Boldface marks the
lowest of the three supervised selectors in each row. Anchoring wins at the smallest budget in all four settings
and in seven of the remaining nineteen comparisons; its advantage concentrates at ten
and twenty-five labels and fades from fifty onward. Runs are $45$ per CNN
architecture, pooling three executions, $17$ on ViT--DomainNet and $27$ on
ViT--CIFAR-20.}
\label{tab:sweep}
\begin{tabular}{@{}llcccc@{}}
\toprule
Setting & $n$ & Dist.\ & Val-CE & CE-combo & Align \\
\midrule
ResNet50--DomainNet & 10 & 0.2198 & 0.1819 & \textbf{0.1664} & 0.1678 \\
 & 25 & 0.2198 & \textbf{0.1015} & 0.1047 & 0.1211 \\
 & 50 & 0.2198 & \textbf{0.0619} & 0.0702 & 0.0882 \\
 & 100 & 0.2198 & \textbf{0.0334} & 0.0372 & 0.0532 \\
 & 150 & 0.2198 & 0.0282 & \textbf{0.0276} & 0.0398 \\
 & 300 & 0.2198 & 0.0132 & \textbf{0.0071} & 0.0121 \\
\midrule
MobileNetV2--DomainNet & 10 & 0.0873 & 0.2270 & 0.1026 & \textbf{0.0875} \\
 & 25 & 0.0873 & 0.0828 & \textbf{0.0440} & 0.0462 \\
 & 50 & 0.0873 & 0.0350 & \textbf{0.0227} & 0.0265 \\
 & 100 & 0.0873 & 0.0130 & \textbf{0.0116} & 0.0129 \\
 & 150 & 0.0873 & \textbf{0.0075} & 0.0089 & 0.0078 \\
 & 300 & 0.0873 & \textbf{0.0012} & 0.0049 & 0.0045 \\
\midrule
ViT--DomainNet & 10 & 0.2843 & 0.2078 & \textbf{0.1395} & 0.2046 \\
 & 25 & 0.2843 & 0.0912 & \textbf{0.0822} & 0.1273 \\
 & 50 & 0.2843 & \textbf{0.0593} & 0.0634 & 0.0860 \\
 & 100 & 0.2843 & \textbf{0.0287} & 0.0343 & 0.0709 \\
 & 150 & 0.2843 & \textbf{0.0165} & 0.0215 & 0.0511 \\
\midrule
ViT--CIFAR-20 & 10 & 0.0955 & 0.2215 & \textbf{0.1155} & 0.1496 \\
 & 25 & 0.0955 & 0.0998 & \textbf{0.0880} & 0.1139 \\
 & 50 & 0.0955 & \textbf{0.0541} & 0.0636 & 0.0821 \\
 & 100 & 0.0955 & \textbf{0.0373} & 0.0550 & 0.0693 \\
 & 150 & 0.0955 & \textbf{0.0270} & 0.0458 & 0.0599 \\
 & 200 & 0.0955 & \textbf{0.0260} & 0.0443 & 0.0517 \\
\bottomrule
\end{tabular}
\end{table}

\begin{table*}[!tb]
\centering
\scriptsize
\setlength{\tabcolsep}{3pt}
\caption{Upper-tail behavior across clean-label budgets, pooling all three
end-to-end executions, $45$ runs per architecture. Left block: $95$th percentile of
target cross-entropy regret pooled over run--subset outcomes rather than
over run means, since the tail is a property of the outcome distribution. Right block: the
catastrophic-selection rate $\Pr(R>0.1)$. The last column is the complete labeled
pool of $300$.}
\label{tab:tail}
\begin{tabular}{llcccccc c cccccc}
\toprule
& & \multicolumn{6}{c}{$95$th-percentile regret} & & \multicolumn{6}{c}{$\Pr(R>0.1)$} \\
\cmidrule(lr){3-8}\cmidrule(lr){10-15}
Setting & Selector & 10 & 25 & 50 & 100 & 150 & 300 & & 10 & 25 & 50 & 100 & 150 & 300 \\
\midrule
RN50--DN & Val-CE & 0.58 & 0.35 & 0.24 & 0.14 & 0.12 & 0.08 & 0.50 & 0.33 & 0.17 & 0.07 & 0.06 & 0.02 \\
 & CE-combo & 0.49 & 0.37 & 0.33 & 0.20 & 0.13 & 0.03 & 0.49 & 0.32 & 0.19 & 0.08 & 0.06 & 0.00 \\
 & Align & 0.48 & 0.41 & 0.34 & 0.28 & 0.18 & 0.05 & 0.50 & 0.36 & 0.27 & 0.13 & 0.09 & 0.00 \\
\midrule
MNV2--DN & Val-CE & 1.18 & 0.61 & 0.13 & 0.06 & 0.04 & 0.00 & 0.30 & 0.15 & 0.06 & 0.02 & 0.01 & 0.00 \\
 & CE-combo & 0.71 & 0.20 & 0.12 & 0.08 & 0.06 & 0.02 & 0.23 & 0.13 & 0.06 & 0.02 & 0.02 & 0.00 \\
 & Align & 0.44 & 0.20 & 0.13 & 0.09 & 0.06 & 0.01 & 0.25 & 0.14 & 0.08 & 0.03 & 0.01 & 0.00 \\
\bottomrule
\end{tabular}
\end{table*}

\paragraph{Paired inference and coefficient stability.}
Supplement Table~\ref{tab:paired} reports the paired run-level tests. Under clean
supervision the selected coefficient is generally away from the upper boundary; under
$40\%$ corruption on MobileNetV2 the grid binds frequently, reaching
$\lambda_{\max}$ in $53\%$ of runs; Supplement Table~\ref{tab:lambda} gives the
distribution of the coefficient and Supplement Figure~\ref{fig:lambda} the regret
curve behind it. On ViT--CIFAR-20 at $\eta=0.4$ the paired
differences are exactly zero because the three supervised selectors select the same
candidate in every run at that corruption level, so the comparison carries no
information there.
\begin{table*}[!tb]
\centering\scriptsize\setlength{\tabcolsep}{4pt}
\caption{Distribution of the cross-validated alignment coefficient at the
complete labeled pool. $\Pr(\lambda{=}0)$ is the frequency with which the procedure
declines to use the supervised direction and reverts to pure distortion;
$\Pr(\lambda{=}\lambda_{\max})$ the frequency of selecting the grid boundary. SD is
across run medians.}
\label{tab:lambda}
\begin{tabular}{llcccc c}
\toprule
Setting & Supervision & Median $\lambda$ & IQR & $\Pr(\lambda{=}0)$
& $\Pr(\lambda{=}\lambda_{\max})$ & SD \\
\midrule
ResNet50--DomainNet & Clean & 7.5 & [5.5, 9.0] & 0.111 & 0.156 & 3.12 \\
 & 20\% noise & 9.0 & [8.5, 9.5] & 0.000 & 0.244 & 1.28 \\
 & 40\% noise & 8.5 & [7.5, 9.0] & 0.000 & 0.067 & 1.01 \\
\midrule
MobileNetV2--DomainNet & Clean & 3.5 & [0.0, 5.5] & 0.289 & 0.022 & 2.81 \\
 & 20\% noise & 4.0 & [0.0, 5.5] & 0.289 & 0.000 & 2.91 \\
 & 40\% noise & 10.0 & [8.0, 10.0] & 0.133 & 0.533 & 3.41 \\
\midrule
ViT--DomainNet & Clean & 5.5 & $[2.0,6.0]$ & 0.059 & 0.000 & 2.63 \\
 & 20\% noise & 7.5 & $[5.5,8.0]$ & 0.000 & 0.059 & 3.23 \\
 & 40\% noise & 7.0 & $[6.0,8.0]$ & 0.000 & 0.059 & 1.32 \\
\midrule
ViT--CIFAR-20 & Clean & 6.5 & $[5.8,7.5]$ & 0.000 & 0.000 & 1.60 \\
 & 20\% noise & 6.5 & $[5.0,8.0]$ & 0.000 & 0.000 & 1.50 \\
 & 40\% noise & 7.5 & $[6.8,7.5]$ & 0.000 & 0.000 & 0.64 \\
\bottomrule
\end{tabular}
\end{table*}

\begin{figure}[h]
\centering
\includegraphics[width=0.98\columnwidth]{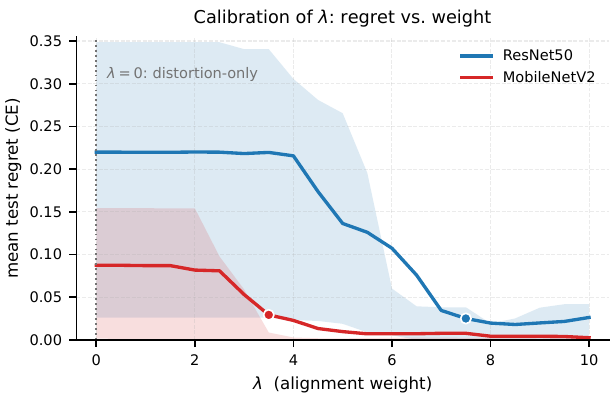}
\caption{Mean target cross-entropy regret as a function of the alignment weight
$\lambda$ on the two DomainNet CNN settings. $\lambda = 0$ recovers
distortion-only selection. Mean regret decreases with $\lambda$ on both
settings, steeply on ResNet50 between roughly $\lambda=4$ and $\lambda=8$ and more
gradually thereafter; the shaded bands are interquartile ranges across runs. Marked
points are the medians of the cross-validated coefficient.}
\label{fig:lambda}
\end{figure}

\paragraph{Corruption at the complete pool.}
Supplement Table~\ref{tab:noise} isolates corruption from scarcity by fixing the budget at
the complete labeled pool. It evaluates the same $n=300$ conditions as the budget--rate
grid but pools all three executions, so its cohort is larger. It reports medians
alongside means and shows most clearly that at $\eta=0.4$ the directional variant
attains the lowest median on both architectures.
\begin{table}[!tb]
\centering
\scriptsize
\setlength{\tabcolsep}{3pt}
\caption{Symmetric label corruption at the complete labeled pool, pooling all three
end-to-end executions, $45$ runs per architecture: median target cross-entropy
regret with mean in parentheses taken over runs after averaging
corruption realizations within each run. Boldface marks the
lowest median in each row.}
\label{tab:noise}
\begin{tabular}{@{}llccc@{}}
\toprule
Setting & $\eta$ & Val-CE & CE-combo & Align \\
\midrule
RN50--DN & 20\% & \textbf{0.064 (0.078)} & 0.064 (0.068) & 0.065 (0.067) \\
 & 40\% & 0.165 (0.231) & 0.127 (0.200) & \textbf{0.102 (0.180)} \\
\midrule
MNV2--DN & 20\% & 0.009 (0.181) & 0.013 (0.105) & \textbf{0.001 (0.015)} \\
 & 40\% & 0.779 (0.831) & 0.706 (0.770) & \textbf{0.393 (0.399)} \\
\bottomrule
\end{tabular}
\end{table}

\paragraph{The directional term and its control.}
Supplement Table~\ref{tab:control} reports the permutation control and the two components
of the alignment statistic. The control's column is the one to read against
distortion: it tracks the anchor to within $0.025$ everywhere, which is what makes the
comparison informative only where reverting to the anchor is costly. The teacher
component is the best of the three variants on ResNet50 at ten labels and the worst
elsewhere at that budget.

\paragraph{Convention of the permutation control.}
The implemented control, stated in Table~\ref{tab:supp_notation} and in the caption of
Table~\ref{tab:control}, permutes the entire teacher-residual direction $e_y-p_f(x)$
across the labeled examples, so the displacement at $x_i$ is scored against the
complete target direction of another example; the main text abbreviates this as
permuting the observed labels. The two constructions differ: permuting labels alone
would leave the teacher component $\langle\delta_g^p(x_i),p_f(x_i)\rangle$ tied to its
input, and Table~\ref{tab:control} shows that this component carries selection signal
of its own, so a labels-only permutation is not a null for the full directional
statistic. The residual-direction permutation is: it destroys the correspondence
between every input and its own target direction while preserving the marginal
distribution and scale of the statistic, and its behaviour in
Table~\ref{tab:control}, tracking pure distortion, is what a proper null predicts.
The question a labels-only permutation would address, namely what the observed-label
part contributes relative to the teacher part, is answered directly by the
$A_{\mathcal S}^{\mathrm{lab}}$/$A_{\mathcal S}^{\mathrm{teach}}$ decomposition
reported in the same table. All permutation results in this paper use the
residual-direction construction.
\begin{table}[!tb]
\centering\scriptsize\setlength{\tabcolsep}{3pt}
\caption{The permutation control and the two components of the directional term, in median
target cross-entropy regret. The control receives the identical score structure,
coefficient grid and cross-validation as alignment, and differs only in that the
teacher-residual directions $e_y-p_f(x)$ are permuted across the labeled examples
before the statistic is formed. Boldface marks the lowest of the three in each row.
CNN rows pool all three executions, $45$ runs per architecture.}
\label{tab:control}
\begin{tabular}{@{}llcccc@{}}
\toprule
Setting & $n$ & Dist.\ & Align & Perm. & Teach. \\
\midrule
ResNet50--DomainNet & 10 & 0.2323 & 0.1870 & 0.2170 & \textbf{0.1445} \\
 & 25 & 0.2323 & 0.1143 & 0.2228 & \textbf{0.1101} \\
 & 50 & 0.2323 & \textbf{0.0635} & 0.2276 & 0.0692 \\
 & 100 & 0.2323 & \textbf{0.0342} & 0.2304 & 0.0432 \\
 & 150 & 0.2323 & \textbf{0.0270} & 0.2320 & 0.0360 \\
 & 300 & 0.2323 & \textbf{0.0024} & 0.2307 & 0.0260 \\
\midrule
MobileNetV2--DomainNet & 10 & 0.0778 & 0.0816 & \textbf{0.0584} & 0.2169 \\
 & 25 & 0.0778 & \textbf{0.0399} & 0.0703 & 0.1458 \\
 & 50 & 0.0778 & \textbf{0.0241} & 0.0685 & 0.0754 \\
 & 100 & 0.0778 & \textbf{0.0049} & 0.0653 & 0.0568 \\
 & 150 & 0.0778 & \textbf{0.0026} & 0.0662 & 0.0314 \\
 & 300 & 0.0778 & \textbf{0.0000} & 0.0653 & 0.0021 \\
\midrule
ViT--DomainNet & 10 & 0.0200 & 0.0501 & \textbf{0.0231} & 0.1957 \\
 & 25 & 0.0200 & 0.0434 & \textbf{0.0213} & 0.0796 \\
 & 50 & 0.0200 & 0.0295 & \textbf{0.0175} & 0.0387 \\
 & 100 & 0.0200 & 0.0183 & 0.0181 & \textbf{0.0176} \\
 & 150 & 0.0200 & 0.0110 & 0.0173 & \textbf{0.0091} \\
\midrule
ViT--CIFAR-20 & 10 & 0.1009 & 0.1325 & \textbf{0.0912} & 0.2162 \\
 & 25 & 0.1009 & \textbf{0.1010} & 0.1056 & 0.1270 \\
 & 50 & 0.1009 & 0.0764 & 0.0966 & \textbf{0.0745} \\
 & 100 & 0.1009 & 0.0715 & 0.0988 & \textbf{0.0538} \\
 & 150 & 0.1009 & 0.0531 & 0.1002 & \textbf{0.0457} \\
 & 200 & 0.1009 & 0.0458 & 0.1006 & \textbf{0.0398} \\
\bottomrule
\end{tabular}
\end{table}

\paragraph{Deployment cost.}
Supplement Table~\ref{tab:cost} gives the price of each memory budget and the spread that
survives inside it. The two columns answer different questions: the excess over the
unconstrained best is what the constraint costs, while the internal spread is what
selection can still recover once the constraint is imposed, and only the first varies
much across settings. Supplement Table~\ref{tab:costsel} then restricts the candidate
set, the candidate search, the coefficient search and the oracle to the admissible
family; the labeled sample is unchanged, so its regrets are measured against the
best admissible candidate rather than the best overall.
\begin{table}[!tb]
\centering
\scriptsize
\setlength{\tabcolsep}{3pt}
\caption{Selection under a deployment-cost constraint. Cost is the weight memory of
a candidate relative to full precision, with layers kept in full precision counted
at their true width. For each budget we report how many of the $72$ candidates are
admissible, its excess over the unconstrained best, and the mean
spread between the worst and best admissible candidate. Means over $45$ runs per architecture.}
\label{tab:cost}
\begin{tabular}{@{}llccc@{}}
\toprule
Setting & Budget & Admissible & Excess over free & Internal spread \\
\midrule
RN50--DN & $\le 10\%$ & 12 & 0.710 & 14.41 \\
         & $\le 16\%$ & 36 & 0.103 & 15.02 \\
         & $\le 22\%$ & 54 & 0.007 & 15.12 \\
         & none       & 72 & 0.000 & 15.13 \\
\midrule
MNV2--DN & $\le 10\%$ & 12 & 0.901 & 10.41 \\
            & $\le 16\%$ & 36 & 0.869 & 10.63 \\
            & $\le 22\%$ & 54 & 0.868 & 10.63 \\
            & none       & 72 & 0.000 & 11.50 \\
\bottomrule
\end{tabular}
\end{table}

\begin{table*}[!tb]
\centering
\scriptsize
\setlength{\tabcolsep}{4.5pt}
\caption{Constrained selection. The candidate set, the coefficient search and the oracle are
all restricted to $\mathcal G_B=\{g:C(g)\le B\}$; the labeled sample is unchanged, so regret is
measured against the best admissible candidate. Pooled over all three executions, $45$ runs per
architecture, lower of the two supervised selectors in boldface.}
\label{tab:costsel}
\begin{tabular}{@{}ll ccc ccc ccc@{}}
\toprule
& & \multicolumn{3}{c}{$n{=}10$} & \multicolumn{3}{c}{$n{=}50$} & \multicolumn{3}{c}{$n{=}300$} \\
\cmidrule(lr){3-5}\cmidrule(lr){6-8}\cmidrule(lr){9-11}
Setting & Budget & Dist.\ & Val-CE & CE-combo & Dist.\ & Val-CE & CE-combo & Dist.\ & Val-CE & CE-combo \\
\midrule
ResNet50--DomainNet & $\le 10\%$ & 0.095 & 0.097 & \textbf{0.090} & 0.079 & 0.043 & \textbf{0.042} & 0.075 & \textbf{0.002} & 0.004 \\
 & $\le 16\%$ & 0.305 & 0.168 & \textbf{0.160} & 0.304 & \textbf{0.040} & 0.056 & 0.304 & \textbf{0.014} & 0.016 \\
 & $\le 22\%$ & 0.269 & 0.184 & \textbf{0.176} & 0.283 & \textbf{0.063} & 0.078 & 0.290 & 0.015 & \textbf{0.009} \\
 & none & 0.220 & 0.182 & \textbf{0.164} & 0.220 & \textbf{0.062} & 0.070 & 0.220 & 0.013 & \textbf{0.007} \\
\midrule
MobileNetV2--DomainNet & $\le 10\%$ & 0.011 & \textbf{0.003} & 0.005 & 0.000 & \textbf{0.000} & \textbf{0.000} & 0.000 & \textbf{0.000} & \textbf{0.000} \\
 & $\le 16\%$ & 0.318 & 0.229 & \textbf{0.193} & 0.167 & 0.012 & \textbf{0.011} & 0.163 & \textbf{0.000} & \textbf{0.000} \\
 & $\le 22\%$ & 0.437 & 0.310 & \textbf{0.308} & 0.348 & \textbf{0.056} & 0.096 & 0.331 & \textbf{0.003} & 0.007 \\
 & none & 0.076 & 0.227 & \textbf{0.102} & 0.086 & 0.035 & \textbf{0.022} & 0.087 & \textbf{0.001} & 0.005 \\
\bottomrule
\end{tabular}
\end{table*}

\section{Degradation Regimes and the Failure of Confidence-Based Estimators}
\label{supp:regimes}

This section of the supplement supports two claims made in Section~4 of the main paper. The
confidences quoted there for ResNet50, $1.000$ for the worst candidate against
$0.408$ for the best, are the medians reported in
Table~\ref{tab:supp_regimes}; and the statement that the two ViT settings
invert the pattern is quantified here over all executions of the four settings
rather than the single run shown in Figure~1.

Section~\ref{subsec:labelfree} attributes the failure of the confidence-based
estimators to the degradation regime rather than to the estimators themselves.
This section gives the evidence for that attribution, quantifies how far
threshold calibration can compensate, and states why we stop short of a
mechanism. All quantities are computed from the stored predictive distributions; no
retraining or requantization is involved.

\paragraph{The two regimes, quantified.}
Table~\ref{tab:supp_regimes} reports for each setting, over all executions, the
median regret of the four confidence-based estimators alongside the diagnostics
that account for it. The inversion described in the body is visible in three
quantities at once. Accuracy and mean confidence across the $72$ candidates are
uncorrelated or inversely related on the CNNs and positively related on the ViTs;
the worst candidate is more confident than the best on the CNNs and less confident
on the ViTs; and the confidence-based estimators select a chance-accuracy
candidate in most runs of each CNN architecture and in none on the ViTs. The rate
is not uniform within the group: on MobileNetV2 the two ATC variants do so in $39$
and $35$ of $45$ runs while average confidence and entropy do so in all $45$.

\begin{table*}[!tb]
\centering
\scriptsize
\setlength{\tabcolsep}{3pt}
\caption{Degradation regimes. Median target cross-entropy regret of the four
confidence-based estimators, with diagnostics. $\rho$ is the median Pearson
correlation between test accuracy and mean confidence across the $72$
candidates within a run. $c_{\text{worst}}$ and $c_{\text{best}}$ are the median
confidences of the worst and best candidate by test cross-entropy. ``Distinct''
is the median fraction of distinct confidence values among collapsed candidates
(accuracy $\le 1.5\times$ chance). CNN rows pool three end-to-end executions.}
\label{tab:supp_regimes}
\begin{tabular}{lrrrrrrrrrr}
\toprule
Setting & Runs & $K$ & AvgConf & Entropy & ATC-MC & ATC-NE & $\rho$ & $c_{\text{worst}}$ & $c_{\text{best}}$ & Distinct \\
\midrule
ResNet50--DomainNet &45 &30 &15.062 &15.062 & 15.062  & 15.062  &$-0.38$ &1.000 &0.408 &34\% \\
MobileNetV2--DomainNet &45 &30 &10.814 &10.814 & 4.886  & 4.696  &$-0.09$ &0.985 &0.534 &100\% \\
ViT--CIFAR-20         & 27 & 20 &  0.124 &  0.167 &  0.100 &  0.099 & $+0.86$ & 0.579 & 0.726 & 100\% \\
ViT--DomainNet        & 17 & 50 &  0.062 &  0.339 &  0.020 &  0.020 & $+0.88$ & 0.523 & 0.761 & 100\% \\
\bottomrule
\end{tabular}
\end{table*}

\begin{table}[!tb]
\centering\scriptsize\setlength{\tabcolsep}{3pt}
\caption{Collapse rate, the fraction of candidates with accuracy at most
$1.5\times$ chance, by bit width and scaling granularity. PT is per-tensor and PC
per-channel. Percentages are over candidates and runs within each setting.}
\label{tab:supp_collapse}
\begin{tabular}{@{}lcccccccc@{}}
\toprule
& \multicolumn{2}{c}{RN50} & \multicolumn{2}{c}{MNV2}
& \multicolumn{2}{c}{ViT--C20} & \multicolumn{2}{c}{ViT--DN} \\
\cmidrule(lr){2-3}\cmidrule(lr){4-5}\cmidrule(lr){6-7}\cmidrule(lr){8-9}
Bits & PT & PC & PT & PC & PT & PC & PT & PC \\
\midrule
2 & 99 & 100 & 100 & 100 & 86 & 96 & 91 & 88 \\
3 & 36 & 16 & 99 & 100 & 33 & 0 & 31 & 0 \\
4 & 33 & 0 & 91 & 90 & 15 & 0 & 17 & 0 \\
5 & 0 & 0 & 87 & 39 & 0 & 0 & 0 & 0 \\
6 & 0 & 0 & 0 & 0 & 0 & 0 & 0 & 0 \\
8 & 0 & 0 & 0 & 0 & 0 & 0 & 0 & 0 \\
\bottomrule
\end{tabular}
\end{table}

\paragraph{Threshold calibration mitigates without repairing.}
The two CNN settings differ in how far calibration helps, and the difference
tracks confidence saturation. On ResNet50 the collapsed candidates take only
$34\%$ distinct confidence values and reach exactly $1.000$; the ATC threshold
therefore falls inside a mass of ties, its accuracy estimate loses
discriminative power, and the median regret of all four estimators is
approximately the median family spread. On MobileNetV2 confidences remain distinct,
calibration roughly halves the median regret ($10.81 \rightarrow 4.89$), and no
estimator selects the single worst candidate. ATC-MC nevertheless selects a
chance-accuracy candidate in $39$ of $45$ runs and ATC-NE in $35$ of $45$.
Threshold calibration addresses the symptom, not the cause.

\paragraph{Where collapse lives in the candidate grid.}
Table~\ref{tab:supp_collapse} breaks the collapse rate down by bit width and
scaling granularity. One regularity is common to all four settings: collapse concentrates at two and three bits. Granularity matters only at intermediate
widths and not uniformly: per-tensor collapses more often than per-channel at three and
four bits on ResNet50 and at five bits on MobileNetV2, while at two bits the two are
comparable and on ViT--CIFAR-20 the order reverses. MobileNetV2 is the exception in degree
rather than in kind: it still collapses at four and five bits, so a
substantial part of its nominal $72$-candidate family consists of degenerate
predictors. This does not affect reported regrets, which are measured against
the best candidate actually attained, but it qualifies what ``$72$ candidates''
means in that setting.

\paragraph{We do not claim a mechanism.}
The natural explanation for the split in Table~\ref{tab:supp_regimes} is
normalization: the CNNs carry frozen BatchNorm statistics that become stale when
weights are quantized, whereas LayerNorm normalizes at inference and cannot
inherit stale statistics. Our design cannot support that claim, for two reasons.
First, architecture and teacher provenance are confounded: every BatchNorm model
here is obtained by transfer learning against hard labels, and every LayerNorm
model comes from a temperature-based distillation pipeline (distillation
temperatures $2.0$ and $10.0$) whose outputs are consistent with the different
soft-target training regime, although the objective imposes no hard confidence
bound. The maximum confidence observed among collapsed
two-bit candidates is $0.8008$ on ViT--CIFAR-20 and $0.7998$ on ViT--DomainNet,
against $1.0000$ on ResNet50 and $0.9998$ on MobileNetV2, a ceiling
consistent with soft-target training rather than with LayerNorm. Second, the
BatchNorm account does not separate the two CNNs: ResNet50 shows a confidence
gap of $+0.704$ between collapsed and healthy candidates, MobileNetV2 only
$+0.069$, although both carry BatchNorm, and MobileNetV2 collapses more often
while saturating less. Separating the two would require a model that
breaks the pairing, a LayerNorm backbone trained against hard labels or a
convolutional one distilled with soft targets, which is a different construction
from the candidate families this paper selects over. We therefore report the
phenomenon and its consequence, and leave the mechanism open.

\paragraph{Consequence for practice.}
The asymmetry matters more than the average. Confidence-based selection is
adequate in one regime and catastrophic in the other, and deciding which regime
holds requires the target labels the setting denies: the diagnostics in
Table~\ref{tab:supp_regimes} that separate the regimes (accuracy of the
worst candidate, the accuracy--confidence correlation) are themselves
label-dependent. Teacher distortion is unaffected in either regime, because a
collapsed candidate is far from the teacher's output distribution however
confident it is. This is the sense in which anchoring to the teacher is a
safer default rather than merely a better average performer.

\section{Cohort Inclusion Criteria}
\label{supp:artifact}
A run enters the final cohort only if it yields a complete $72$-candidate family
with usable stored predictions. The criteria are structural and were applied to
the runs available during the audit, before any comparative selector result
from those pipelines could enter the tables. A run is excluded if its candidate family is
incomplete, if its per-run outputs are uniform across candidates, or if every
candidate sits at chance accuracy.

One run falls outside it. Of the $18$ ViT--DomainNet runs generated, one
contains $47$ of the $72$ candidates; the missing
indices $47$ to $71$ form a contiguous block at the end of the enumeration
order, so the exclusion cannot depend on measured performance. It does not enter
any table.

The cohort reported throughout is therefore $45$ runs per CNN architecture and
$17$ on ViT--DomainNet and $27$ on ViT--CIFAR-20, as
Table~\ref{tab:supp_runs} states: the
full design of Section~4, and the evidence behind every claim we report.

\def\isChecklistMainFile{}
\clearpage
\makeatletter
\@ifundefined{isChecklistMainFile}{
  \newif\ifreproStandalone
  \reproStandalonetrue
}{
  \newif\ifreproStandalone
  \reproStandalonefalse
}
\makeatother

\ifreproStandalone
\documentclass[letterpaper]{article}
\usepackage[submission]{aaai2027}
\setlength{\pdfpagewidth}{8.5in}
\setlength{\pdfpageheight}{11in}
\frenchspacing

\begin{document}
\fi
\setlength{\leftmargini}{20pt}
\makeatletter\def\@listi{\leftmargin\leftmargini \topsep .5em \parsep .5em \itemsep .5em}
\def\@listii{\leftmargin\leftmarginii \labelwidth\leftmarginii \advance\labelwidth-\labelsep \topsep .4em \parsep .4em \itemsep .4em}
\def\@listiii{\leftmargin\leftmarginiii \labelwidth\leftmarginiii \advance\labelwidth-\labelsep \topsep .4em \parsep .4em \itemsep .4em}\makeatother

\setcounter{secnumdepth}{0}
\renewcommand\thesubsection{\arabic{subsection}}
\renewcommand\labelenumi{\thesubsection.\arabic{enumi}}

\newcounter{checksubsection}
\newcounter{checkitem}[checksubsection]

\newcommand{\checksubsection}[1]{%
  \refstepcounter{checksubsection}%
  \paragraph{\arabic{checksubsection}. #1}%
  \setcounter{checkitem}{0}%
}

\newcommand{\checkitem}{%
  \refstepcounter{checkitem}%
  \item[\arabic{checksubsection}.\arabic{checkitem}.]%
}
\newcommand{\question}[2]{\normalcolor\checkitem #1 #2 \color{blue}}
\newcommand{\ifyespoints}[1]{\makebox[0pt][l]{\hspace{-15pt}\normalcolor #1}}

\ifreproStandalone
\end{document}
\fi

\end{document}